%% file: main.tex
\documentclass[11pt]{article}

\input{header}
\usepackage{macros}

\title{
Hierarchical Linkage Clustering\\ Beyond Binary Trees and Ultrametrics
}

\author{
Maximilien Dreveton, Matthias Grossglauser, Daichi Kuroda, Patrick Thiran 
}
\date{}

\begin{document}

\maketitle
\begin{center}
\begin{tabular}{c}
 École Polytechnique Fédérale de Lausanne (EPFL) 
\end{tabular}
\texttt{\{maximilien.dreveton,matthias.grossglauser,daichi.kuroda,patrick.thiran\}@epfl.ch}
\end{center}

\begin{abstract}
\input{abstract.tex}

\end{abstract}

\tableofcontents

\input{text}

\clearpage

\appendix
\input{appendix}

\bibliographystyle{alpha}
\bibliography{biblio.bib}

\end{document}

%% file: abstract.tex
Hierarchical clustering seeks to uncover nested structures in data by constructing a tree of clusters, where deeper levels reveal finer-grained relationships. Traditional methods, including linkage approaches, face three major limitations: (i) they always return a hierarchy, even if none exists, (ii) they are restricted to binary trees, even if the true hierarchy is non-binary, and (iii) they are highly sensitive to the choice of linkage function. 
In this paper, we address these issues by introducing the notion of a valid hierarchy and defining a partial order over the set of valid hierarchies. We prove the existence of a finest valid hierarchy, that is, the hierarchy that encodes the maximum information consistent with the similarity structure of the data set. In particular, the finest valid hierarchy is not constrained to binary structures and, when no hierarchical relationships exist, collapses to a star tree. 
We propose a simple two-step algorithm that first constructs a binary tree via a linkage method and then prunes it to enforce validity. We establish necessary and sufficient conditions on the linkage function under which this procedure exactly recovers the finest valid hierarchy, and we show that all linkage functions satisfying these conditions yield the same hierarchy after pruning. Notably, classical linkage rules such as single, complete, and average satisfy these conditions, whereas Ward’s linkage fails to do so. 


%% file: text.tex

\section{Introduction}

Hierarchical clustering is the task of inferring a tree of clusters, where deeper levels of the hierarchy uncover finer-grained structures. Unlike flat clustering, which partitions the dataset into $k$ groups, hierarchical clustering identifies nested clusters and captures their containment relationships. A hierarchy is represented by a rooted tree whose leaves correspond to individual clusters, and internal vertices to the merge of smaller clusters into larger overarching groups. 

Hierarchical clustering algorithms typically fall into two main categories: divisive (\textit{top-down}) and agglomerative (\textit{bottom-up}). Top-down algorithms (such as bisection $k$-means or recursive sparsest cut) construct the hierarchy by recursively splitting the data into two parts, working from the root downward~\citep{calinski1974dendrite,rousseeuw2009finding}. In contrast, bottom-up algorithms, such as linkage, iteratively merge the two most similar clusters and recompute the similarity between the newly formed cluster with the non-merged ones, thus building the hierarchy from the bottom upward~\citep{florek1951liaison,johnson1967hierarchical,Miyamoto2018}. 
A key advantage of hierarchical clustering is that it does not require specifying the number of clusters in advance, which is a significant benefit when the true structure of the data is unknown, and provides the user with a finer representation of the data than flat clustering. 

Nonetheless, these hierarchical clustering algorithms have a major drawback: the inferred tree representing the hierarchy is always binary. Indeed, top-down approaches recursively split the data into two parts, whereas bottom-up methods merge clusters in pairs. This is problematic, as these algorithms (i) cannot recover non-binary hierarchies and (ii) always return a hierarchy, even when the data does not exhibit one. This can in particular lead to infer hierarchical structures that do not exist. Finally, (iii) the outcome of agglomerative algorithms is highly sensitive to the choice of linkage criterion, which defines how cluster similarities are computed. Indeed, at each iteration, the similarity between the newly formed cluster and the remaining ones can be defined in several ways, such as the minimum (single linkage), maximum (complete linkage), or average (average linkage) pairwise distance between their elements. Different linkage functions can yield drastically different trees on the \textit{same} dataset, with no principled way to determine which is more appropriate. Together, these issues raise fundamental questions about the robustness and interpretability of traditional hierarchical clustering. 

The purpose of this paper is to overcome these limitations. To address them, we enlarge the hypothesis class from binary hierarchies to the class of all hierarchical tree structures (allowing arbitrary branching, including the degenerate “star‐tree” structure). This expansion ensures realizability: any true underlying hierarchy can now be represented in the model class. However, this expansion also dramatically increases the combinatorial complexity of the search space: while there are already $\Theta(4^{k-1})$ binary trees with $k$ labeled leaves,\footnote{More precisely, there are $C_{k-1}$ binary trees with $k$ labeled leaves, where $C_k = \frac{1}{k} \binom{2k}{k}$ is the $k$-th Catalan number, and Stirling's formula provides $C_k \sim 4^k k^{-3/2} \pi^{-1/2}$.} there are~$k^{k-2}$ non-binary trees (by Cayley's formula), a significantly larger number.


To navigate through this huge space of all possible hierarchies, we introduce the notion of a \emph{valid hierarchy} over a finite set~$\cX$ of items equipped with a similarity measure $s\colon \cX \times \cX \to \R_+$. We define a natural partial order over valid hierarchies and show that this partially ordered set admits a unique greatest element. Because this greatest element dominates all other valid hierarchies under the partial order, we refer to it as the \emph{finest valid hierarchy}, as it encodes the maximum information compatible with the similarity function. Importantly, this finest valid hierarchy is not necessarily binary. Moreover, in the absence of any meaningful hierarchical structure, the finest valid hierarchy degenerates into a star tree, where all elements connect directly to the root, and thus correctly captures the absence of a meaningful hierarchy. 

Moreover, we observe that the finest valid hierarchy extends the classical correspondence between ultrametrics and dendrograms \citep{johnson1967hierarchical,bocker1998recovering,carlsson2010characterization,cohen2019hierarchical}. Indeed, when the similarity function $s$ arises from an ultrametric, the finest valid hierarchy coincides with the unique dendrogram representing that ultrametric. 
Thus, our framework generalizes this well-known equivalence to the broader class of arbitrary similarities for which no additional metric or ultrametric assumption is made. We discuss this point in more details in Section~\ref{subsection:ultrametrics_dendrograms}. 

To construct the finest valid hierarchies, we propose a two-step algorithm. In the first step, we build a binary tree using a linkage method. Because the resulting tree may not define a valid hierarchy, the second step prunes it by removing the internal vertices that violate validity, thereby removing incorrect splits. Unlike traditional linkage methods, this two-step bottom-up procedure (i) does not necessarily yield a binary tree, and (ii) defaults to a star tree when the data contains no hierarchical structure. Consequently, it provides a principled and robust approach to hierarchical clustering that simultaneously addresses the shortcomings of classical linkage-based methods. 

We establish necessary and sufficient conditions on the linkage function under which this two-step procedure exactly recovers the finest valid hierarchy. Importantly, these conditions imply that all linkage functions satisfying them yield the same pruned hierarchy, regardless of their specific merging criteria. This result resolves one of the major drawbacks of agglomerative clustering, namely its strong dependence on the choice of linkage. Crucially, we also show that classical linkage rules such as single, complete, and average linkage satisfy these necessary and sufficient conditions, whereas Ward’s linkage does not. This highlights a sharp conceptual distinction between these methods with important practical implications for real-world applications.

 Moreover, many axiomatic approaches to hierarchical clustering have sought to characterize linkage rules from first principles, and a recurring outcome of these analyses is the special status of single linkage. In particular, when enforcing natural invariance and consistency axioms, single linkage often emerges as the unique admissible rule (see, for example, \cite{jardine1968construction,zadeh2009uniqueness,carlsson2010characterization} and the discussion in Section~\ref{section:discussion}). However, despite its appealing theoretical properties, single linkage is well known to suffer from the chaining effect, where clusters may grow by successively linking distant points through intermediate ones. This behavior often produces elongated, “caterpillar-like’’ trees that poorly capture compact cluster structure.
Our results show that pruning eliminates this degeneracy. More broadly, our validity-based framework does not privilege any particular linkage function: any linkage rule satisfying the necessary and sufficient conditions of Theorem~\ref{theorem:algo_recover} recovers the same finest valid hierarchy after pruning. In this way, pruning unifies the behavior of linkage methods within a single theoretical framework.

\paragraph{Notations}
In this paper, $k$ is an integer and $\cX = \{ x_1, \cdots, x_k\}$ is a discrete set of $k$ items. We denote $[k] = \{1,\cdots, k\}$. We denote by $\lca_{T}(x, y)$, or simply $\lca(x,y)$, the least common ancestor between two leaves $x$ and~$y$ of a tree $T$. Finally, $\powerset(\cX)$ is the powerset of $\cX$.

\paragraph{Paper Organization} The paper is organized as follows. Section~\ref{sec:hierarchy} introduces the main definitions and notations. Section~\ref{section:finestValidHierarchy} establishes the existence of the finest valid hierarchy and presents a method to construct it using a linkage-based pruning algorithm. Finally, Section~\ref{section:discussion} discusses related work and positions our contributions within the state-of-the-art. 

\section{Preliminaries and Notations}
\label{sec:hierarchy}

\subsection{Partial Order among Hierarchies} 

Consider a discrete set $\cX = \{ x_1, \cdots, x_k\}$ of $k$ items and denote by $\cT(\cX)$ the set of trees whose leaves are $\{x_1\}, \cdots, \{x_k\}$. For convenience, an element of $\cT(\cX)$ is represented by a set of sets; it will be made precise in the following definition. We provide several examples in Figure~\ref{fig:example_trees}, for illustration purposes.  

\begin{definition}
\label{definition:set_representation_tree}
A tree $T$ on a finite set $\cX = \{ x_1, \cdots, x_k \}$ is a collection of nonempty subsets $T \subseteq \powerset(\cX) \backslash \{ \emptyset \}$ such that: 
\begin{enumerate}
 \item (\textbf{Laminarity}) for all $u,v \in T$, we have $u \cap v \in \{ \emptyset, u, v \}$;
 \item (\textbf{Connectedness}) all singletons and the full set belong to $T$:  $\{x_i\} \in T$ for all $i \in [k]$, and $\cX \in T$.
\end{enumerate}
\end{definition}
Each $t \in T$ is called a vertex of the tree $T \in \cT(\cX)$. 
Condition 1 (laminarity) ensures that any two vertices are either disjoint or nested, and the inclusion relation among them defines the ancestor–descendant structure of the tree. For example, if $t_1 \subseteq t_2$ and there is no $t_3$ with $t_1 \subseteq t_3 \subseteq t_2$, then $t_2$ is the parent of $t_1$. More generally, every non-leaf vertex~$t$ of a tree~$T$ is the union of at least two subsets that are the children of $t$. Condition~2 additionally requires that the root of $T \in \cT(\cX)$ must be $\cX$, and that the leaves of~$T$ are $\{x_1\}, \cdots, \{x_k\}$. As a result, $T$ is connected. As an example, we provide in Figure~\ref{fig:example_trees} three trees $T_1$, $T_2$ and $T_3$ belonging to $\cT(\{x_1, \cdots, x_5\})$.

\begin{figure}[!ht]
\begin{subfigure}{0.32\textwidth}
\centering
\begin{tikzpicture}
  [level distance=10mm,
  scale=1,
   every node/.style={draw=blue, inner sep=1pt},
   level 1/.style={sibling distance=15mm,nodes={}},
   level 2/.style={sibling distance=12mm,nodes={}},
   level 3/.style={sibling distance=7mm,nodes={}}]
  \node { $\{x_1,x_2,x_3,x_4,x_5\}$ }
     child {node { $\{x_1, x_2, x_3\}$ }
         child {node { $\{x_1\}$} }
         child {node { $\{x_2\}$} }
         child {node { $\{x_3\}$} }
       }
         child {node { $\{x_4\}$}}
         child {node { $\{x_5\}$ }}
     ;
\end{tikzpicture}
\caption{ $T_1$ }
\end{subfigure}
\hfill
\begin{subfigure}{0.32\textwidth}
\centering
\begin{tikzpicture}
  [level distance=10mm,
  scale=1,
   every node/.style={draw=blue, inner sep=1pt},
   level 1/.style={sibling distance=25mm,nodes={}},
   level 2/.style={sibling distance=10mm,nodes={}},
   level 3/.style={sibling distance=5mm,nodes={}}]
  \node { $\{x_1,x_2,x_3,x_4,x_5\}$ }
       child {node { $\{x_1, x_2, x_3\}$ }
         child {node { $\{x_1\}$ } }
         child {node { $\{x_2\}$ } }
         child {node { $\{x_3\}$ } }
       }
       child {node[draw=green] { $\{x_4,x_5\}$ }
         child {node { $\{x_4\}$ }}
         child {node { $\{x_5\}$ }}
       };
\end{tikzpicture}
\caption{ $T_2$ }
\end{subfigure}
\hfill
\begin{subfigure}{0.32\textwidth}
\centering
\begin{tikzpicture}
  [level distance=10mm,
  scale=1,
   every node/.style={draw=blue, inner sep=1pt},
   level 1/.style={sibling distance=25mm,nodes={}},
   level 2/.style={sibling distance=10mm,nodes={}},
   level 3/.style={sibling distance=5mm,nodes={}}]
  \node { $\{x_1,x_2,x_3,x_4,x_5\}$ }
       child {node[draw=red] { $\{x_1,x_2\}$ }
         child {node { $\{x_1\}$ } }
         child {node { $\{x_2\}$ } }
       }
       child {node[draw=red] { $\{x_3,x_4,x_5\}$ }
         child {node { $\{x_3\}$ } }
         child {node { $\{x_4\}$ }}
         child {node { $\{x_5\}$ }}
       };
\end{tikzpicture}
\caption{ $T_3$ }
\end{subfigure}
\caption{Three trees $T_1$, $T_2$ and $T_3$ belonging to $\cT(\{x_1,x_2,x_3,x_4,x_5\})$. 
In terms of Definition~\ref{definition:set_representation_tree}, these three trees are explicitly written as
(a) $T_1 = \{ \{x_1\}, \{x_2\}, \{x_3\}, \{x_4\}, \{x_5\}, \{x_1,x_2,x_3\}, \{x_1,x_2,x_3,x_4,x_5\} \}$, \\
(b) $T_2 = \{ \{x_1\}, \{x_2\}, \{x_3\}, \{x_4\}, \{x_5\}, \{x_1,x_2,x_3\}, \{x_4,x_5\}, \{x_1,x_2,x_3,x_4,x_5\} \}$, and 
(c) $T_3 = \{ \{x_1\}, \{x_2\}, \{x_3\}, \{x_4\}, \{x_5\}, \{x_1,x_2\}, \{x_3,x_4,x_5\}, \{x_1,x_2,x_3,x_4,x_5\} \}$. 
}
\label{fig:example_trees}
\end{figure}
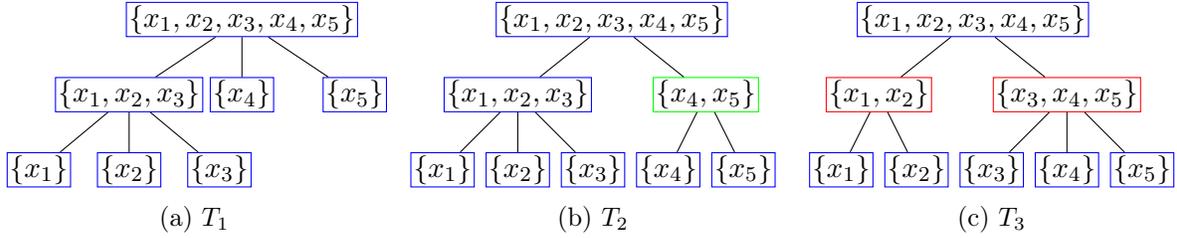

\subsection{Valid Clusters and Valid Hierarchies}

Consider a function $s \colon \cX \times \cX \to \R_+$ that measures the similarity between pairs of items in the finite set $\cX = \{x_1, \dots, x_k\}$. We only assume that (i) $s$ is symmetric, \textit{i.e.} $s(x,y) = s(y,x)$ for all $x,y \in \cX$, and that (ii) self-similarity dominates pairwise similarity, \textit{i.e.} $s(x,x) > s(x,y)$ whenever $x \neq y$.
 
We refer to any subset of $\cX$ as a \emph{cluster}. A cluster is said to be \emph{valid} with respect to the similarity function~$s$ if the similarity between any two items inside the cluster is strictly greater than their similarity with any item outside the cluster.
 
\begin{definition}[Valid Cluster]
\label{def:valid_clusters}
A subset $C \subseteq \cX$ is a \emph{valid cluster} with respect to~$s$ if
\[
\forall x,y \in C \: \forall z \in \cX \setminus C \colon \quad s(x,y) > s(x,z). 
\]
\end{definition}

In particular, all singletons $\{x_1\}, \dots, \{x_k\}$ are valid clusters, as well as the full set $\cX$ itself. This notion of validity has appeared in prior work under different names: in~\cite{balcan2008discriminative} it is referred to as \emph{strict separation}, 
while in~\cite{ackerman2014incremental} it is called \emph{nice clustering}. A more detailed discussion of these connections can be found in Section~\ref{subsec:related_axiomatic}.

Before extending the notion of validity from individual clusters to entire hierarchies, it is useful to note a simple structural property: two valid clusters cannot partially overlap. This property ensures that the collection of all valid clusters forms a laminar family, which will later justify viewing it as a valid hierarchy.

\begin{lemma}[Validity implies laminarity]
\label{lemma:validity_imply_laminar}
Let $C_1, C_2 \subseteq \cX$ be two valid clusters with respect to~$s$. Then $C_1 \cap C_2 \in \{  \emptyset, C_1, C_2 \}$. 
\end{lemma}

\begin{proof}
Assume, for contradiction, that $C_1$ and $C_2$ overlap but that neither contains the other. 
Then there exist $a \in C_1 \cap C_2$, $b \in C_1 \setminus C_2$, and $c \in C_2 \setminus C_1$. 
Because $C_1$ is valid, we have $s(a,b) > s(a,c)$, while the validity of $C_2$ yields $s(a,c) > s(a,b)$, a contradiction. Hence the two sets $C_1$, $C_2$ satisfy the laminarity property. 
\end{proof}

 
 Next, observe that any tree $T \in \cT(\cX)$ defines a hierarchy among its leaves $\{x_1\}, \cdots, \{x_k\}$: from the bottom up, leaves of $T$ merge into branches, and these branches merge further up into larger branches until they reach the root. Each vertex $t$ of a tree is a cluster, which may or may not be valid (albeit the root $\{x_1,\cdots,x_k\}$ and the leaves $\{x_1\}, \cdots, \{x_k\}$ of any tree are always valid). We say that the hierarchy defined by a tree~$T \in \cT(\cX)$ is \new{valid} with respect to the similarity function~$s$ if all the vertices of $T$ are valid clusters. This leads to the following definition. 

\begin{definition} 
\label{def:validHierarchies}
Consider a finite set $\cX = \{ x_1, \cdots, x_k \}$, a similarity function $s \colon \cX \times \cX \to \R_+$, and a tree $T \in \cT(\cX)$. $T$ is a \new{valid hierarchy} over $\cX$ with respect to $s$ iff every vertex $t$ of $T$ is a valid cluster. In other words, we must have 
\begin{align}
 \forall t \in T \colon \quad \min_{ \substack{ x,y \in t, \\ z \in \cX \backslash t } } \, s(x,y) - s(x,z) \ > \ 0.
 \label{eq:hierarchical_tree}
\end{align} 
We denote by $\setOfValidHierarchy(\cX, s)$ the set of valid hierarchies over $\cX$ with respect to $s$. 
\label{def:hierarchical_tree}
\end{definition}
 This definition imposes that the similarity between two distinct items $x$ and $y$ that are close in the hierarchy defined by~$T$ is strictly larger than the similarity with a third item $z$ located further away in the hierarchy. More precisely, the validity condition imposes that for all triplets of items $x,y,z \in \cX$ such that $\lca(x,y) \subsetneq \lca(x,z)$, we have $s(x, y) > \max\{ s(x, z), s(y, z)\}$. 

The following example shows that multiple trees can qualify as valid hierarchies for a given similarity measure.  

\begin{example}
\label{exa:valid_trees}
 Consider the trees $T_1$, $T_2$ and $T_3$ given in Figure~\ref{fig:example_trees}, and the similarity function~$s$ defined over $\{x_1,\cdots, x_5\}$ and represented by the following matrix 
 \begin{align*}
   s \weq  
   \begin{pmatrix}
3 & 2 & 2 & 1 & 1 \\
2 & 3 & 2 & 1 & 1 \\
2 & 2 & 3 & 1 & 1 \\
1 & 1 & 1 & 3 & 2 \\
1 & 1 & 1 & 2 & 3
 \end{pmatrix}.  
 \end{align*}
 Then $T_3$ is not a valid hierarchy with respect to $s$. Indeed, let $t = \{x_1,x_2\} \in T_3$. Let $x = x_1$, $y = x_2$ and $z = x_3$. Then $s(x,y) - s(x,z) = 2 - 2 = 0$, hence Condition~\eqref{eq:hierarchical_tree} is not verified. In contrast, $T_1$ and $T_2$ are two valid hierarchies. 
\end{example}

\subsection{Partial Order between Hierarchies}
\label{subsec:partial_order}

The previous sections introduced the basic objects of interest: trees, clusters, and valid hierarchies. We now formalize how hierarchies can be compared. 

\begin{definition}[Partial order on $\cT(\cX)$]
For two trees $T, T' \in \cT(\cX)$, we write
\[
T \subseteq T'
\quad \Longleftrightarrow \quad
\forall\, t \in T \colon  t \in T'.
\]
In words, $T \subseteq T'$ means that every cluster contained in $T$ also appears in $T'$. 
\end{definition}

The inclusion relation $\subseteq$ is reflexive, antisymmetric, and transitive, and thus defines an order on the set of trees~$\cT(\cX)$. Continuing the example of Figure~\ref{fig:example_trees}, we have $T_1 \subseteq T_2$, and we say that $T_2$ \textit{contains} $T_1$. Because neither $T_3 \subseteq T_1$ nor $T_1 \subseteq T_3$ holds, the relationship $\subseteq$ is only a \textit{partial order}. Finally, we write $T \subsetneq T'$ if $T'$ contains $T$ but $T' \ne T$. 

This order $\subseteq$ naturally restricts to the subset of valid hierarchies $\cH(\cX,s) \subseteq \cT(\cX)$, allowing us to compare hierarchies according to the refinement of their cluster families. Intuitively, if $T \subseteq T'$, then $T'$ represents a \emph{finer} hierarchical structure, as it preserves all clusters of $T$ while potentially adding new intermediate groupings. The greatest element of this order will later be shown to be the finest valid hierarchy $\trueHierarchy$.

\section{Finest Valid Hierarchy}
\label{section:finestValidHierarchy}

\subsection{Existence of the Finest Valid Hierarchy}
\label{subsec:valid_hierarchies}

Using the convention that $\min(\emptyset) = \infty$, we observe that the star tree $T_0$, where all leaves $\{x_1\}, \cdots, \{x_k\}$ connect directly to the root $\{x_1,\cdots, x_k\}$, verifies Condition~\eqref{eq:hierarchical_tree} for \textit{any} similarity function $s$. This implies that the set $\setOfValidHierarchy(\cX, s)$ of valid hierarchy is never empty. Nonetheless, because it always belongs to $\setOfValidHierarchy(\cX, s)$, irrespectively of the similarity function $s$, $T_{0}$ does not bring any non-trivial information about the structure induced by the similarity function $s$ over the set of items~$\cX$. Therefore, when the star tree~$T_0$ is the only valid hierarchy, it indicates an \textit{absence of hierarchy} among~$\cX$. 

We are, however, left with the following question: \textit{If there are several trees defining valid hierarchies, which one should be preferred?} To answer this question, first observe that any tree $T \neq T_0$ defining a valid hierarchy should be preferred over the star tree~$T_0$, because~$T$ offers more insights on the hierarchical structure of~$\cX$ than~$T_0$. Moreover, given a tree $T \in \setOfValidHierarchy(\cX,s)$ defining a valid hierarchy, we notice that any tree $T' \subseteq T$ also defines a valid hierarchy. Indeed, the set of inequalities implied by the valid hierarchy $T'$ is included in the set of inequalities implied by~$T$. Therefore, we should prefer~$T$ over any $T' \subseteq T$. For instance, the trees~$T_1$ and~$T_2$ in Figure~\ref{fig:example_trees} are both valid with respect to the similarity function given in Example~\ref{exa:valid_trees}, but~$T_2$ captures more details about the hierarchy than~$T_1$. Another implication of $T_1 \subseteq T_2$ is that $T_2$ being a valid hierarchy implies that~$T_1$ is also valid; the converse is not true.\footnote{To see that, consider for example the vertex $\{4,5\} \in T_2 \backslash T_1$. The fact that~$T_1$ is a valid hierarchy does not imply that for any $z \in \{1,2,3\}$, $s(4,5) > \min\left(s(4,z), s(5, z) \right)$, which is required to establish that~$T_2$ is a valid hierarchy.}  

However, the relationship $\subseteq$ defines only a \textit{partial order} over the set of valid hierarchies $\setOfValidHierarchy(\cX,s)$. We could therefore encounter two valid hierarchies $T$ and $T'$, such that none contains the other. The following theorem solves this issue by establishing that the set of valid hierarchies admits a greatest element $\trueHierarchy(\cX,s)$ for the partial order $\subseteq$. Recall that in a partially ordered set, the greatest element, if it exists, is always unique and by definition satisfies $T \subseteq \trueHierarchy(\cX,s)$ for any valid hierarchy $T \in \setOfValidHierarchy(\cX,s)$. Therefore, the tree~$\trueHierarchy(\cX,s)$ is \textit{the} valid hierarchy that contains the most information about the hierarchical structure of $\cX$, and we call $\trueHierarchy(\cX,s)$ the \emph{finest valid hierarchy} of $\cX$ with respect to $s$. 

\begin{theorem}[Existence of the Finest Valid Hierarchy]
 \label{theorem:mostInformativeHierarchy} 
 Let $\trueHierarchy(\cX,s)$ be the set of all valid clusters. 
 Then, $\trueHierarchy(\cX,s)$ is the greatest element of the partially ordered set $(\cH(\cX,s),\subseteq)$. 
\label{theorem:unique_hierarchy}
\end{theorem}

\begin{proof}
Let $\trueHierarchy = \trueHierarchy(\cX,s)$ be the set of all valid clusters. 

By definition, any valid hierarchy $T \in \cH(\cX,s)$ consists only of valid clusters, hence $T \subseteq \trueHierarchy$. Therefore, $\trueHierarchy$ is an upper bound of $(\cH(\cX,s),\subseteq)$.

By Lemma~\ref{lemma:validity_imply_laminar}, valid clusters form a laminar family, and as all singletons and $\cX$ itself are valid, $\trueHierarchy$ satisfies the conditions of Definition~\ref{definition:set_representation_tree}; thus $\trueHierarchy \in \cT(\cX)$. Because $\trueHierarchy$ is a tree and consists only of valid clusters, $\trueHierarchy \in \cH(\cX,s)$. 

Therefore, $\trueHierarchy$ is both an element of $\cH(\cX,s)$ and an upper bound of $\cH(\cX,s)$, making it the greatest (and hence unique) element of 
$(\cH(X,s),\subseteq)$.
\end{proof}

Theorem~\ref{theorem:unique_hierarchy} shows that the finest valid hierarchy is exactly composed of all the valid clusters. We can give another interpretation of the finest valid hierarchy. Denote by $|T|$ the number of vertices of a tree $T$. Because $T' \subseteq T$ implies that $|T'| \le |T| $, we have 
 \begin{align*}
  \trueHierarchy(\cX,s) \weq \argmax_{T \in \cH(\cX,s) } \, |T|, 
 \end{align*}
 and the argmax must be a singleton as the greatest element is unique. When no confusion is possible, we omit the mention of $\cX$ and $s$ by simply writing $\trueHierarchy$ instead of $\trueHierarchy(\cX,s)$ and by calling~$\trueHierarchy$ the finest valid hierarchy.

\subsection{Construction of the Finest Valid Hierarchy}
\label{sec:hierarchy_recovery}

We now demonstrate how to construct the finest valid hierarchy. 


\subsubsection{Linkage Algorithms}

In this section, we present a simple algorithm for recovering the finest valid hierarchy. Traditionally, hierarchy recovery is performed using variants of \emph{linkage}. A linkage method is specified by a rule that updates the similarities between clusters as they are iteratively merged. Starting from the partition of~$\cX$ into the singletons ($\{x_1\}, \cdots, \{x_k\}$), at each step the two most similar clusters are merged, and the similarity between the resulting cluster and the remaining ones is determined by the chosen rule.
For instance, suppose that the items $\{x_1\}$ and $\{x_2\}$ are merged at the first step. The linkage rule then prescribes how to compute the similarity between the new cluster $\{x_1, x_2\}$ and the remaining clusters $\{x_a\}$ for $a \in \{3,\dots,k\}$. 

There is a large variety of linkage rules. Moreover, different linkage rules typically produce different hierarchical structures, and the choice of the rule has a direct impact on the recovered hierarchy. To capture this impact more generally, we assume a \emph{general linkage process} defined by an update function $f \colon \R^3_+ \times \N^3 \to \R$. When two clusters $t_1$ and $t_2$ are merged to form $t_1 \cup t_2$, the similarity between $t_1 \cup t_2$ and any other cluster $t_3$ is computed as
\begin{align}
s(t_1 \cup t_2, t_3) \weq f\big(s(t_1,t_2), s(t_2,t_3), s(t_3,t_1), |t_1|, |t_2|, |t_3|\big).
\label{eq:general_update_formula}
\end{align}
Thus, the update depends only on the pairwise similarities among $t_1,t_2,t_3$ and on their sizes $|t_1|,|t_2|,|t_3|$. Observe that because $t_1 \cup t_2 = t_2 \cup t_1$, the function $f$ must satisfy 
\[
f\big(s(t_1,t_2), s(t_2,t_3), s(t_3,t_2), |t_1|, |t_2|, |t_3|\big) 
\weq f\big(s(t_2,t_1), s(t_1,t_3), s(t_3,t_1), |t_2|, |t_1|, |t_3|\big).
\]
Moreover, we will also assume that $f$ is continuous with respect to its first coordinate. 
\begin{assumption}
The linkage update function $f$ is right-continuous with respect to its first argument: for any $(q_1,q_2,q_3,n_1,n_2,n_3) \in \R^3_+ \times \N^3$, we have
\[
\lim_{\delta \to 0^+} f(q_1+\delta,q_2,q_3,n_1,n_2,n_3) \weq f(q_1,q_2,q_3,n_1,n_2,n_3).
\]
\label{assumption:f_r_cont}
\end{assumption}
Algorithm~\ref{algo:linkage} describes this general linkage algorithm. We note that the update rule~\eqref{eq:general_update_formula} is closely related to the classical Lance–Williams formula~\citep{murtagh2017algorithms}, which provides a unifying scheme for many hierarchical clustering methods. The Lance–Williams formula expresses the new similarity as a linear combination of $s(t_1,t_3)$, $s(t_2,t_3)$, and $s(t_1,t_2)$ with coefficients depending only on the cluster sizes. Our formulation generalizes this formula by allowing~$f$ to be an arbitrary function of the same quantities, and thus includes all Lance–Williams linkages as well as potentially broader classes of linkage rules. We refer to Section~\ref{subsection:consequences_theorem_trimming_recovery} for a more detailed discussion and examples.

\begin{algorithm}[!ht]
\caption{Linkage}
\label{algo:linkage}
\KwInput{Items $\cX = \{x_1, \dots, x_k\}$, similarity function $s \colon \cX \times \cX \rightarrow \R_+$, update rule $f \colon \R^3 \times \N^3 \to \R$}

Initialize hierarchy $\Tlinkage = \{ \{x_1\}, \dots, \{x_k\} \}$ and active clusters $\cC_{active}^{(0)} = \{ \{x_1\}, \dots, \{x_k\} \}$.

\For{$m = 1$ \KwTo $k-1$}{
Select ${t_1,t_2} \in \argmax_{t_1' \neq t_2' \in \cC_{active}^{(m-1)}} s(t_1',t_2')$ \label{line_algo_argmax}

 Form $t = t_1 \cup t_2$, update $\Tlinkage \leftarrow \Tlinkage \cup \{t\}$

Update active clusters: $\cC_{active}^{(m)} \leftarrow \big(\cC_{active}^{(m-1)} \setminus \{t_1, t_2\} \big) \cup \{t\}$

For each $t_3 \in \cC_{active}^{(m)} \setminus \{t\}$, define $s(t_1 \cup t_2, t_3)$ using~\eqref{eq:general_update_formula}. 
}
\KwReturn{$\Tlinkage$}
\end{algorithm}

\subsubsection{Trimmed Linkage}

Because Algorithm~\ref{algo:linkage} merges items strictly in pairs, it always produces a binary tree. Consequently, if the finest valid hierarchy~$\trueHierarchy$ is \emph{not} binary, Algorithm~\ref{algo:linkage} cannot directly recover it. More generally, the resulting binary hierarchy~$\Tlinkage$ is not guaranteed to be valid and may contain clusters that violate the validity condition given in Definition~\ref{def:valid_clusters}. We address this by trimming the vertices of~$\Tlinkage$ that violate the validity condition. This algorithm yields a valid hierarchy and is formalized in Algorithm~\ref{algo:merging_vertices}. 

\begin{algorithm}[!ht]
\caption{Trimmed Linkage}
\label{algo:merging_vertices}
\KwInput{Items $\cX = \{x_1, \dots, x_k\}$, similarity function $s \colon \cX \times \cX \rightarrow \R_+$, update rule $f \colon \R^3 \times \N^3 \to \R$}

Apply Algorithm~\ref{algo:linkage} with input $(\cX, s, f)$ to obtain~$\Tlinkage$

Denote $\Tlinkage^- = \left\{ t \in \Tlinkage \colon \min_{ (x, y, z) \in t \times t \times \left(\cX \setminus t \right)}  s(x,y) - s(x,z)  \le 0 \right\}$ \label{algo:line_changed}

\KwReturn{$\Ttrimmedlinkage = \Tlinkage \backslash \Tlinkage^-$}
\end{algorithm}

\begin{lemma}
Denote by $\Tlinkage = \Tlinkage(\cX,s,f)$ and by $\Ttrimmedlinkage = \Ttrimmedlinkage(\cX, s, f)$ the output of Algorithm~\ref{algo:linkage} and Algorithm~\ref{algo:merging_vertices}, respectively. Then, $\Ttrimmedlinkage$ is a valid hierarchy, and $\Ttrimmedlinkage = \Tlinkage \cap \trueHierarchy$. 
\label{lemma:trim_linkage}
\end{lemma}

\begin{proof}
Pruning removes exactly those clusters that are invalid, so $\Ttrimmedlinkage$ contains only the valid clusters of $\Tlinkage$. In particular, $\Ttrimmedlinkage$ contains the root $\cX$ and the leaves $\{x_i\}$. The laminarity of the remaining clusters follows from Lemma~\ref{lemma:validity_imply_laminar}, ensuring that $\Ttrimmedlinkage$ is a tree (and hence is indeed a valid hierarchy). Finally, because $\trueHierarchy$ contains all the valid clusters (Theorem~\ref{theorem:mostInformativeHierarchy}), we also have $\Ttrimmedlinkage = \Tlinkage \cap \trueHierarchy$.
\end{proof}

By construction, Algorithm~\ref{algo:merging_vertices} always produces a valid hierarchy $\Ttrimmedlinkage$, which is thus contained in the finest one. Because $\Ttrimmedlinkage = \Tlinkage \cap \trueHierarchy$, this hierarchy $\Ttrimmedlinkage$ will be equal to the finest valid one $\trueHierarchy$ if and only if $\trueHierarchy \subseteq \Tlinkage$. In the following, we establish that this holds if and only if the update rule $f$ extending $s$ by~\eqref{eq:general_update_formula} satisfies the following two conditions.

\begin{condition}[Monotonicity across merges]
\label{condition:linkage_contain_1}
For any similarity function $s$, the update rule $f$ extends~$s$ such that for any subsets $t_1, t_2, t_3, t_4 \subset \cX$, we have 
 \begin{align*}
  \forall i \in \{1,2\} \colon s(t_i,t_3)> \max \left(s(t_i,t_4),s(t_3,t_4)\right) \ \Rightarrow \ s(t_1 \cup t_2, t_3) > \max\left(s(t_1 \cup t_2, t_4), s(t_3,t_4)\right).
 \end{align*}
\end{condition}

\begin{condition}[Dominance preservation]
\label{condition:linkage_contain_2}
For any similarity function $s$, the update rule $f$ extends $s$ such that for any subsets $t_1, t_2, t_3, t_4 \subset \cX$, we have 
\begin{align*}
    s(t_1, t_2) > \max_{ \substack{i \in \{1,2\} \\ j \in \{3,4 \} } } s(t_i, t_j) \ \Rightarrow \ s(t_1, t_2) > \max_{i \in \{1,2\} }s(t_i, t_3 \cup t_4).
\end{align*}
\end{condition}

Condition~\ref{condition:linkage_contain_1} ensures that merging two clusters that are both more similar to a third cluster does not weaken this similarity. Indeed, this condition ensures that merging clusters~$t_1$ and~$t_2$ preserves an existing stronger similarity with a third cluster~$t_3$ relative to a fourth cluster~$t_4$. Similarly, Condition~\ref{condition:linkage_contain_2} ensures that clusters that are strongly similar should remain more similar to each other than to any newly formed cluster that merges weaker connections. Indeed, if two clusters $t_1$ and $t_2$ are more similar to each other than to $t_3$ and~$t_4$, then Condition~\ref{condition:linkage_contain_2} maintains this similarity after $t_3$ and~$t_4$ are merged. 


\begin{lemma}
\label{lemma:linkage_contain_iff} 
Denote by $\Tlinkage = \Tlinkage(\cX,s,f)$ the tree returned by Algorithm~\ref{algo:linkage} with input $(\cX,s,f)$. We have 
\begin{align*}
 \forall (\cX,s) \colon \trueHierarchy(\cX,s) \subseteq \Tlinkage \iff f \text{ satisfies Conditions~\ref{condition:linkage_contain_1} and~\ref{condition:linkage_contain_2}.} 
\end{align*}
\end{lemma}

\begin{proof}[Proof sketch]
The full proof is lengthy and is given in Appendix~\ref{appendix:proof_trimmed_linkage_recovers}, and we summarize the key steps below. The equivalence splits into two complementary statements. 

For sufficiency, assume that $f$ satisfies Conditions~\ref{condition:linkage_contain_1} and~\ref{condition:linkage_contain_2}. We show by induction on the linkage step that every valid cluster $C$ (\textit{i.e.,} every vertex of $\trueHierarchy$) is formed at some point of the iterative merging process. At the initial step, the forest (the incomplete linkage tree that is in construction) consists of singletons (formed by each item), each of which is either contained in $C$ or disjoint from $C$. 
By Conditions~\ref{condition:linkage_contain_1} and~\ref{condition:linkage_contain_2}, the similarity between any two clusters contained in~$C$ is strictly higher than the similarity between any cluster contained in~$C$ and any cluster in $\cX\backslash C$. This ensures that clusters contained in~$C$ always merge with each other, internally in~$C$, before any cluster outside~$C$ can become part of a merge. Hence, the algorithm eventually produces~$C$. However, non-valid clusters can be created during the iterative process, and thus $\Tlinkage$ is not necessarily a valid hierarchy. 

For necessity, assume~$f$ violates one of the conditions. Then, we explicitly build a similarity function $s_1$ over a subset of items $\cX_1$ for which the linkage order produced by~$f$ omits a valid cluster that should belong to $\trueHierarchy(\cX_1,s_1)$. In other words, any violation of either condition yields a concrete counterexample where the linkage tree fails to contain the true hierarchy. 
\end{proof}

 The following theorem simply follows from Lemmas~\ref{lemma:trim_linkage} and~\ref{lemma:linkage_contain_iff}.

\begin{theorem}[Characterization of Linkages Recovering the Finest Valid Hierarchy]
\label{theorem:algo_recover}
Denote by $\Ttrimmedlinkage(\cX, s, f)$ the tree returned by Algorithm~\ref{algo:merging_vertices} with input $(\cX,s,f)$. The following holds:
\begin{align*}
 \forall (\cX,s) \colon \ \Ttrimmedlinkage(\cX,s,f) \weq \trueHierarchy(\cX,s) \iff f \text{ satisfies Conditions~\ref{condition:linkage_contain_1} and~\ref{condition:linkage_contain_2}}.
\end{align*}
\end{theorem}

 
\cite[Theorem 2]{balcan2008discriminative} established that the hierarchy produced by single linkage contains all valid clusters. Lemma~\ref{lemma:linkage_contain_iff} generalizes this result by providing necessary and sufficient conditions on the linkage rule that guarantee the inferred hierarchy contains exactly all valid clusters. Moreover, although \cite{balcan2008discriminative} employs a hierarchical procedure, their goal is to recover a flat ground-truth valid partition. In contrast, our focus is on recovering the entire hierarchical structure that includes all valid clusters while excluding any invalid ones. The exclusion of invalid clusters is done pruning is crucial to obtain Theorem~\ref{theorem:algo_recover}.

\subsection{Consequences of Theorem~\ref{theorem:algo_recover}}
\label{subsection:consequences_theorem_trimming_recovery}

Theorem~\ref{theorem:algo_recover} shows that the pruning procedure of Algorithm~\ref{algo:merging_vertices} is \emph{consistent} with respect to the finest valid hierarchy: whenever the linkage update rule~$f$ satisfies Conditions~\ref{condition:linkage_contain_1} and~\ref{condition:linkage_contain_2}, the algorithm exactly recovers $\trueHierarchy(X,s)$ for any similarity function~$s$. In particular, all linkage rules meeting these conditions produce the \textit{same} pruned hierarchy.
Hence, these conditions are \emph{fundamental} for hierarchy recovery: the outcome of hierarchical clustering is \emph{invariant} to the choice of linkage function, as long as the update rule $f$ satisfies Conditions~\ref{condition:linkage_contain_1} and~\ref{condition:linkage_contain_2}. This eliminates one of the major sources of arbitrariness in traditional agglomerative clustering, where different linkage rules can yield incompatible trees on the same dataset.

\paragraph{Examples of linkage satisfying Conditions~\ref{condition:linkage_contain_1} and~\ref{condition:linkage_contain_2}}

The formulation~\eqref{eq:general_update_formula} together with Assumption~\ref{assumption:f_r_cont} subsumes the following three classical linkage rules as special cases: 
\begin{itemize}
\item \textit{Single linkage:}
$f(s_{12}, s_{23}, s_{13}, n_1, n_2, n_3) = \max\{s_{23}, s_{13} \}$.
\item \textit{Complete linkage:}
$f(s_{12}, s_{23}, s_{13}, n_1, n_2, n_3) = \min\{ s_{23}, s_{13} \}$.
\item \textit{Weighted Average linkage:}
$f(s_{12}, s_{23}, s_{13}, n_1, n_2, n_3) = \frac{n_1}{n_1+n_2} s_{13} + \tfrac{n_2}{n_1+n_2} s_{23}$.
\item \textit{Unweighted Average linkage:}
$f(s_{12}, s_{23}, s_{13}, n_1, n_2, n_3) =  \frac{s_{13} + s_{23} }{ 2 }$.
\end{itemize}
The next lemma shows these linkages all satisfy the Conditions~\ref{condition:linkage_contain_1} and~\ref{condition:linkage_contain_2}.

\begin{lemma}
\label{lemma:standardLinkage_satisfy_conditions}
 Single, complete, weighted average, and unweighted average linkages satisfy Conditions~\ref{condition:linkage_contain_1} and~\ref{condition:linkage_contain_2}. 
\end{lemma}

\paragraph{Examples of linkage violating Conditions~\ref{condition:linkage_contain_1} and~\ref{condition:linkage_contain_2}}
Some of the other most widely used linkage methods are Ward’s minimum variance~\citep{ward1963hierarchical}, Median, and Centroid linkages~\cite{murtagh2017algorithms}. These linkages are usually defined in terms of a dissimilarity measure rather than a similarity. Let $d \colon \cX \times \cX \to \R_+$ be a dissimilarity function such that $d(x,x) < d(x,y)$ for all $x \ne y \in \cX$ (note that $d$ is not necessarily a distance). A valid hierarchy with respect to $d$ can be defined analogously to~\eqref{eq:hierarchical_tree} as
 \begin{align*}
   \forall t \in T \colon \quad \min_{ \substack{ x, y \in t, \\ z \in \cX \backslash t } } \, d(x,z) - d(x,y) \ > \ 0.
 \end{align*}
With this modification, Theorems~\ref{theorem:unique_hierarchy} and~\ref{theorem:algo_recover} hold after replacing the $\argmax$ in line~\ref{line_algo_argmax} of Algorithm~\ref{algo:linkage} by an $\argmin$, together with the two following suitably adapted conditions. 

\begin{condition}[Monotonicity across merges -- dissimilarity version]
\label{condition:linkage_contain_1_dissimilarity}
For any dissimilarity function $d$, the update rule $f$ extends~$d$ such that for any subsets $t_1, t_2, t_3, t_4 \subset \cX$, we have 
 \begin{align*}
  \forall i \in \{1,2\} \colon d(t_i,t_3)> \min\left(d(t_i,t_4),d(t_3,t_4)\right) \ \Rightarrow \ d(t_1 \cup t_2, t_3) < \min\left(d(t_1 \cup t_2, t_4), d(t_3,t_4)\right).
 \end{align*}
\end{condition}

\begin{condition}[Dominance preservation -- dissimilarity version]
\label{condition:linkage_contain_2_dissimilarity}
For any similarity function $s$, the update rule $f$ extends $s$ such that for any subsets $t_1, t_2, t_3, t_4 \subset \cX$, we have 
\begin{align*}
    d(t_1, t_2) < \min_{ \substack{i \in \{1,2\} \\ j \in \{3,4 \} } } d(t_i, t_j) \ \Rightarrow \ d(t_1, t_2) < \min_{i \in \{1,2\} }d(t_i, t_3 \cup t_4).
\end{align*}
\end{condition}

For Ward, Median, and Centroid linkages, the dissimilarity update rule can be expressed via the \emph{Lance–Williams formula}~\citep{murtagh2017algorithms}:
\begin{align*}
d(t_1 \cup t_2, t_3) \ = \ \eta_1 d(t_1, t_3) + \eta_2 d(t_2, t_3) + \beta d(t_1,t_2) + \gamma |d(t_1,t_3) - d(t_2,t_3) |,
\end{align*}
where the value of the coefficients $(\eta_1, \eta_2, \beta, \gamma)$ depend on the linkage method. Their values are summarized in Table~\ref{tab:lance_williams_coeffs}. 

\begin{table}[!ht]
\centering
\caption{Lance--Williams coefficients for the linkages mentioned in the text.}
\label{tab:lance_williams_coeffs}
\begin{tabular}{lccc}
\toprule
\textbf{Linkage} & $\eta_1$ & $\eta_2$ & $\beta$ \\
\midrule
Ward & $\dfrac{|t_1|+|t_3|}{|t_1|+|t_2|+|t_3|}$ 
     & $\dfrac{|t_2|+|t_3|}{|t_1|+|t_2|+|t_3|}$ 
     & $-\dfrac{|t_3|}{|t_1|+|t_2|+|t_3|}$ \\
Median & $\dfrac{|t_1|}{|t_1|+|t_2|}$ 
       & $\dfrac{|t_2|}{|t_1|+|t_2|}$ 
       & $-\dfrac{|t_1||t_2|}{(|t_1|+|t_2|)^2}$ \\
Centroid & $\tfrac{1}{2}$ 
       & $\tfrac{1}{2}$ 
       & $-\tfrac{1}{4}$ \\
\bottomrule
\end{tabular}
\end{table}
The following lemma shows that these linkages fail to satisfy the dissimilarity versions of our key conditions. As a result, these linkages can produce, after pruning, a hierarchy that do not contain all valid clusters. We provide an explicit numeric example for Ward linkage in Appendix~\ref{subsec:counterexamples_other_linkages}. 

\begin{lemma}
\label{lemma:someLinkage_break_conditions}
Ward linkage does not satisfy Condition~\ref{condition:linkage_contain_1_dissimilarity}, while Median and Centroid linkages do not satisfy Condition~\ref{condition:linkage_contain_2_dissimilarity}.
\end{lemma}

\section{Discussion and Related Work}
\label{section:discussion}

\subsection{Ultrametrics and Dendrograms}
\label{subsection:ultrametrics_dendrograms}

 \paragraph{From ultrametrics to dendrograms.}
A dissimilarity $d \colon \cX \times \cX \to \R_+$ is an \emph{ultrametric} if, for every distinct $x,y,z\in \cX$,
\begin{align}
d(x,y) \wle \max\{d(x,z), d(y,z)\}.
\label{eq:ultrametric_ineq}
\end{align}
This ``strong triangle inequality'' implies in particular that the largest value among the three distances $\{d(x,y), d(x,z), d(y,z)\}$ occurs at least twice. Consequently, every triplet $(x,y,z)$ admits a consistent \emph{closeness relation}: exactly one pair is no farther apart than the other two, so that each triangle with corners $x,y,z$ is isosceles. These local relations collectively define a family of nested clusters, and there exists a unique rooted tree $T_d$ with leaf set $\cX$ and an associated height function $h\colon T_d \to\R_+$ such that
\[
d(x,y) \weq h(\lca_{T_d}(x,y)),
\]
where $\lca_{T_d}(x,y)$ denotes the lowest common ancestor of $x$ and $y$ in $T_d$.

The pair $(T_d,h)$ is called a \emph{dendrogram} (or a dated tree).  Conversely, every dendrogram induces an ultrametric via this construction. Hence the space of ultrametrics on $\cX$ is in bijection with the space of (real-valued) dendrograms on~$\cX$, and this relationship has been extensively investigated in the literature \citep{johnson1967hierarchical,hartigan1967representation,benzecri1984analyse,rammal1986ultrametricity,carlsson2010characterization,cohen2019hierarchical}. 

\paragraph{From ultrametricity to validity.}
Replacing dissimilarities by similarities converts the ultrametric condition~\eqref{eq:ultrametric_ineq} into
\begin{equation}
\label{eq:ultrametric_ineq_similarity}
s(x,y) \ge \min\{s(x,z), s(y,z)\}, \qquad \forall\, x,y,z \in \cX.
\end{equation}
This inequality expresses the ultrametric property in terms of
similarities: among any three points, the smallest similarity value
occurs at least twice.

The following lemma (whose proof is straightforward and deferred to Appendix~\ref{subsec:proof_prop_when_ultrametric}) demonstrates that when~$s$ satisfies~\eqref{eq:ultrametric_ineq_similarity}, the finest valid hierarchy $\trueHierarchy(\cX,s)$ can be equipped with a height function to form a dendrogram. This shows that, in the ultrametric case, the finest valid hierarchy coincides with the unique dendrogram associated with $s$. 

 \begin{lemma} \label{lemma:when_ultrametric}
 Let $s$ be an ultrametric similarity function and let $\trueHierarchy = \trueHierarchy(\cX,s)$ be the finest valid hierarchy with respect to $s$. Then, there exists a function $h \colon \trueHierarchy(\cX,s) \to \R_+$ satisfying $h(t) > h(t')$ for any $t \subsetneq t'$ and such that 
 \begin{align*}
  \forall (x, y) \in \cX \times \cX \colon \quad s(x, y) \weq h(\lca_{\trueHierarchy}(x, y)). 
 \end{align*}
 \end{lemma}

Imposing the ultrametric condition~\eqref{eq:ultrametric_ineq_similarity} is, however, very restrictive. It enforces a global consistency among all triplet comparisons, which rarely holds in practice. In contrast, Condition~\eqref{eq:hierarchical_tree} is much weaker: it defines valid hierarchies for an arbitrary similarity function~$s$, requiring only that the self-similarity of an item with itself exceeds its similarity with any other item. This broader definition allows our framework to handle non-ultrametric similarity functions while still producing a well-defined hierarchical structure.

To illustrate the contrast, consider the star tree~$T_0$ in Figure~\ref{fig:ultrametric_similarity_star_tree}. If $s$ is constrained to be ultrametric, then $T_0$ is the dendrogram associated with $s$
if and only if $s(x,y) = p\,\mathbf{1}(x=y) + q\,\mathbf{1}(x\ne y)$
for some constants $p>q$. This constraint is highly restrictive. Without the ultrametric constraint, the same star tree naturally represents the absence of hierarchical structure in a similarity function~$s$, corresponding to a similarity matrix with no meaningful nested organization. The same reasoning applies to more complex examples (see Figure~\ref{fig:ultrametric_similarity_nonstar_tree}).

\paragraph{Beyond ordered similarities: Symbolic ultrametrics and future directions.}
A related and elegant generalization of ultrametrics replaces real-valued distances by symbolic ones. A \emph{symbolic ultrametric} \citep{bocker1998recovering} is a symmetric map $\delta\colon \cX\times\cX \to \Sigma$ taking values in an arbitrary set of symbols~$\Sigma$ (where no order exists in $\Sigma$), such that for every distinct triple $(x,y,z)\in\cX^3$ the multiset $\{\delta(x,y),\delta(x,z),\delta(y,z)\}$ contains at most two distinct symbols. 
This purely combinatorial constraint is both necessary and sufficient for $\delta$ to be representable as a \emph{symbolically dated tree}: there exists a unique rooted tree whose internal nodes are labeled by symbols so that $\delta(x,y)$ equals the label of the lowest common ancestor of $x$ and~$y$. Hence symbolic ultrametrics stand in a one-to-one correspondence with symbolically labeled dendrograms, paralleling the classical bijection between real-valued ultrametrics and weighted trees. 

This correspondence suggests an intriguing direction for future work. Our validity framework relies on inequality comparisons between similarities, and therefore presupposes an ordering on the similarity values. A natural question is whether a \emph{symbolic characterization of valid hierarchies} could be developed in the same spirit as symbolic ultrametrics---for instance, by encoding for each triplet $(x,y,z)$ a symbolic relation designating which pair among the three is ``closer'', and requiring these relations to be globally consistent with respect to an unobserved ultrametric distance, as done in~\cite{aho1981inferring}. Such a formulation would yield a purely symbolic analogue of our validity notion in Definition~\ref{def:valid_clusters}, extending the link between ultrametrics and hierarchical structures beyond the ordered setting.

\begin{figure}[!ht]
\begin{subfigure}[b]{0.2\textwidth}
\centering
\begin{tikzpicture}
  [level distance=15mm,
   level 1/.style={sibling distance=5mm,nodes={}},
   ]
  \coordinate
    child { node {$x_1$} }
    child {node {$x_2$}}
    child {node {$x_3$}}
    child {node {$x_4$}}
    child {node {$x_5$}}
     ;
\end{tikzpicture}
\caption{Star tree~$T_0$}
\label{fig:ultrametric_similarity_star_tree_tree}
\end{subfigure}
\begin{subfigure}[b]{0.26\textwidth}
\centering
\includegraphics[width=1.0\textwidth]{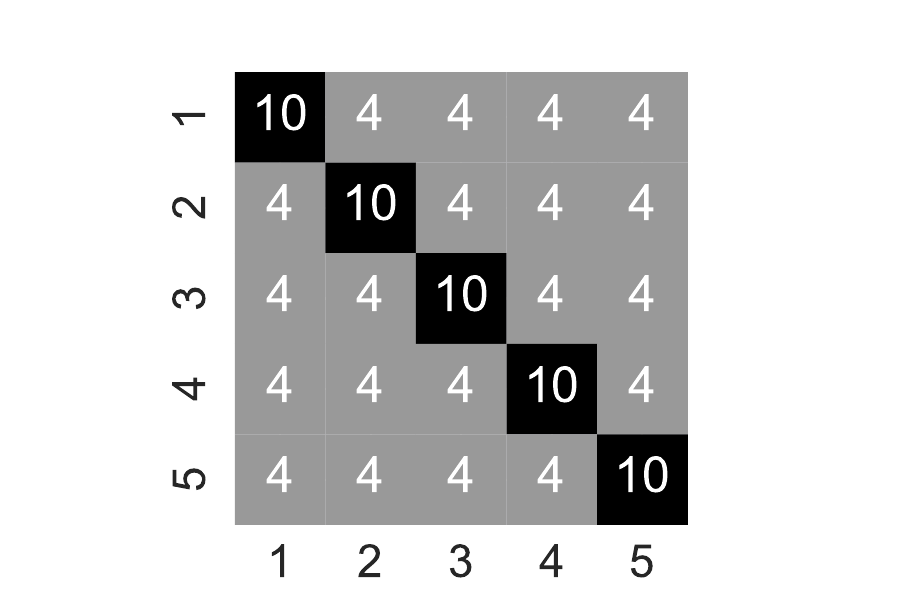}
\caption{$s_1$ (ultrametric)}
\label{fig:ultrametric_similarity_star_tree_ultrametric}
\end{subfigure}
\begin{subfigure}[b]{0.26\textwidth}
 \centering
 \includegraphics[width=\textwidth]{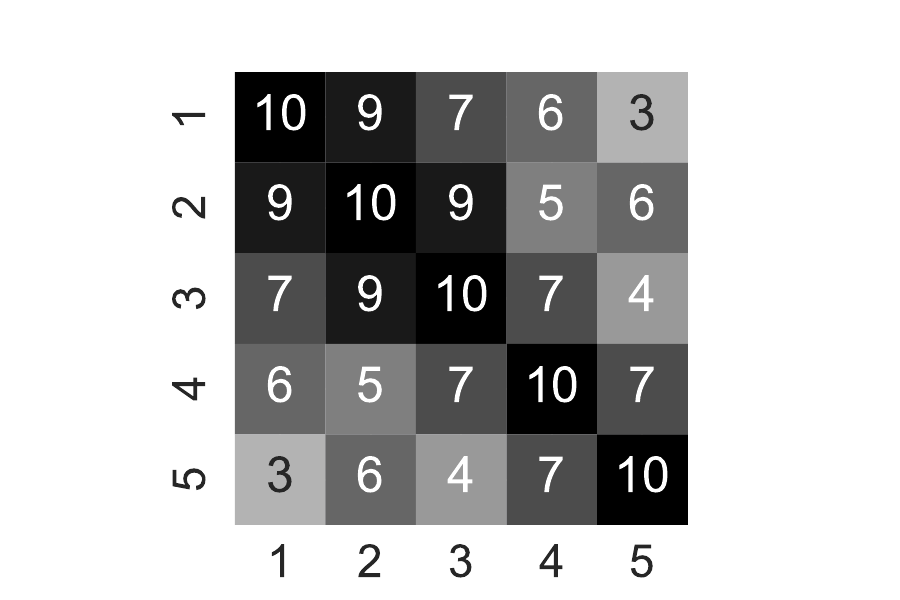}
 \caption{$s_2$ (not ultrametric)}
\label{fig:ultrametric_similarity_star_tree_nonultrametric}
\end{subfigure}
\begin{subfigure}[b]{0.26\textwidth}
 \centering
 \includegraphics[width=\textwidth]{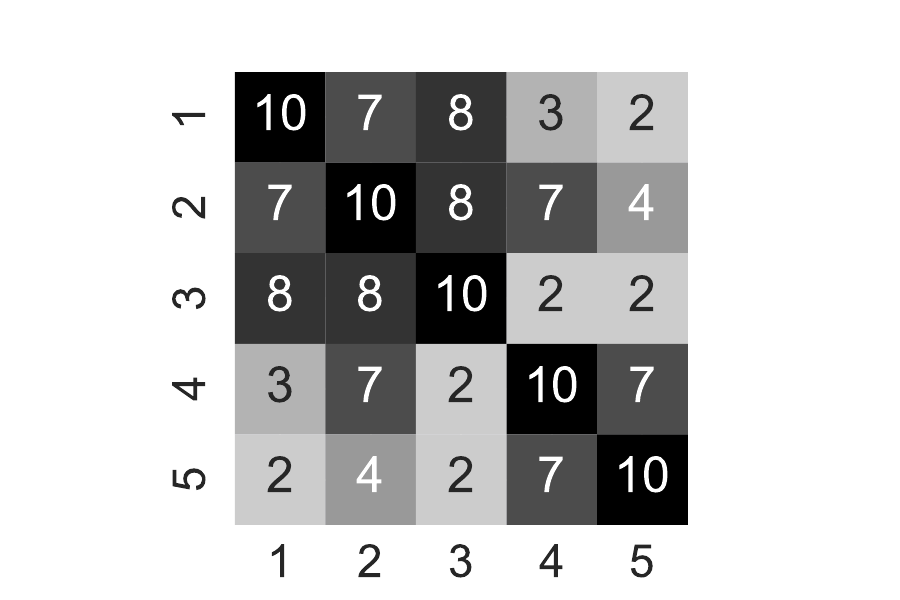}
\caption{$s_3$ (not ultrametric)}
\label{fig:ultrametric_similarity_star_tree_nonultrametric2}
\end{subfigure}
\caption{Ultrametric and non-ultrametric similarities whose true hierarchy is the star tree~$T_0$ given in Figure~\ref{fig:ultrametric_similarity_star_tree_tree}. Figures~\ref{fig:ultrametric_similarity_star_tree_ultrametric}, \ref{fig:ultrametric_similarity_star_tree_nonultrametric}, and~\ref{fig:ultrametric_similarity_star_tree_nonultrametric} provide three similarity functions $s_1$, $s_2$, and $s_3$ such that $\trueHierarchy(\cX,s_1) = \trueHierarchy(\cX,s_2) = \trueHierarchy(\cX,s_3) = T_0$. The similarity function $s_1$ is an ultrametric, whereas~$s_2$ and $s_3$ are not.}
\label{fig:ultrametric_similarity_star_tree}
\end{figure}
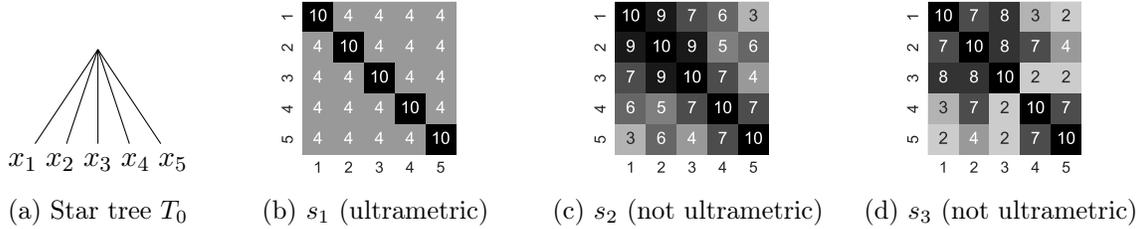

\begin{figure}[!ht]
\begin{subfigure}[b]{0.29\textwidth}
\centering
\begin{tikzpicture}
  [level distance=7mm,
   level 1/.style={sibling distance=17mm,nodes={}},
   level 2/.style={sibling distance=6mm,nodes={}},
   level 3/.style={sibling distance=5mm,nodes={}},
   level 4/.style={sibling distance=3mm,nodes={}}
   ]
  \coordinate
     child {
       child {node {$x_1$}}
       child {node {$x_2$}}
       child {node {$x_3$}}
     }
    child {
        child { 
            child {node {$x_4$} } 
            child {node {$x_5$} }
        }
         child {node {$x_6$} }
         child {node {$x_7$} }
       }
     ;
\end{tikzpicture}
\caption{Tree $T$}
\end{subfigure}
\begin{subfigure}[b]{0.35\textwidth}
\centering
\includegraphics[width=\textwidth]{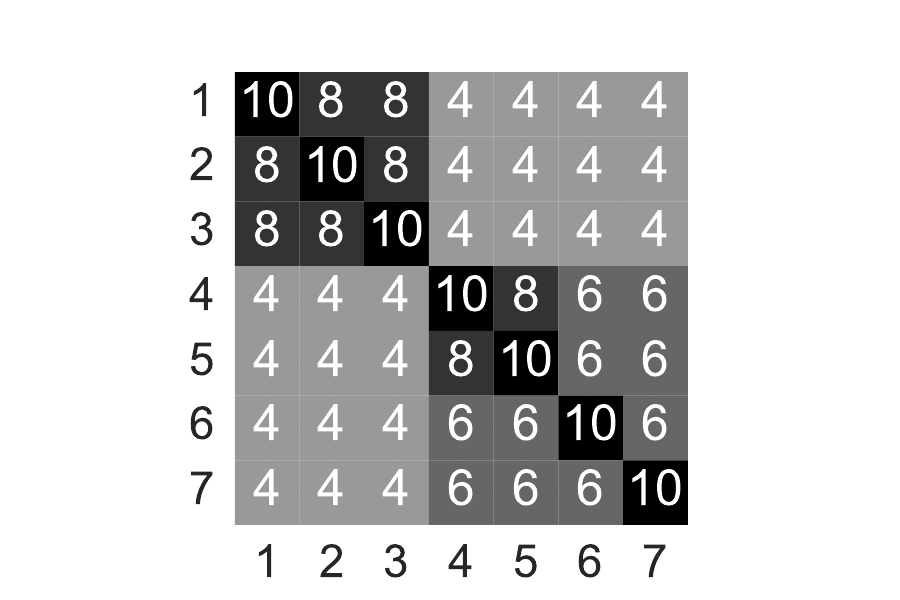}
\caption{$s_1$ (ultrametric)}
\label{fig:ultrametric_similarity_nonstar_tree_ultrametric} 
\end{subfigure}
\begin{subfigure}[b]{0.35\textwidth}
\centering
\includegraphics[width=\textwidth]{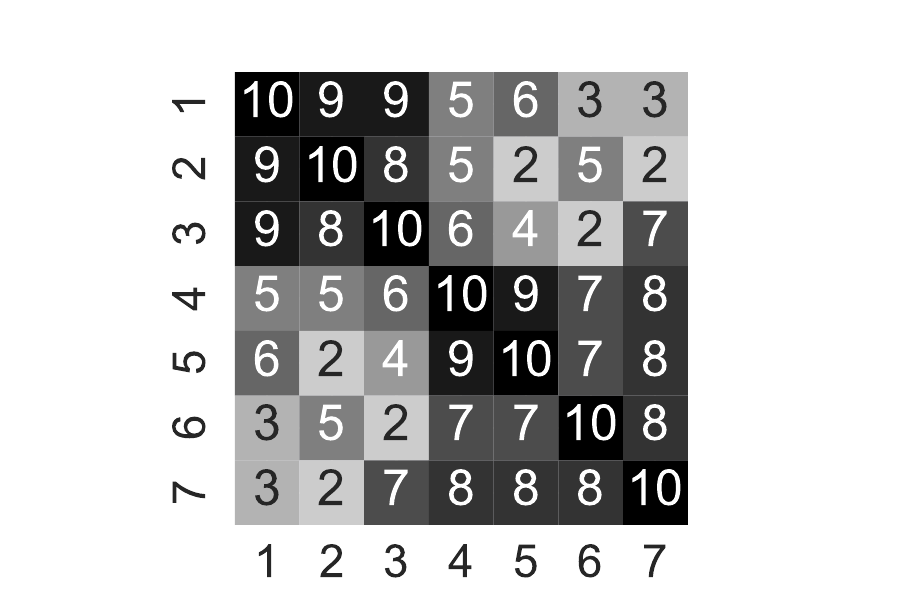}
\caption{$s_2$ (not ultrametric)}
\label{fig:ultrametric_similarity_nonstar_tree_nonultrametric}
\end{subfigure}
\caption{Ultrametric and non-ultrametric similarities whose true hierarchy is the tree $T$ given in Figure~\ref{fig:ultrametric_similarity_nonstar_tree}. Figures~\ref{fig:ultrametric_similarity_nonstar_tree_ultrametric} and~\ref{fig:ultrametric_similarity_nonstar_tree_nonultrametric} provide two similarity functions $s_1$ and $s_2$ such that $\trueHierarchy(\cX,s_1) = \trueHierarchy(\cX,s_2) = T$. The similarity function $s_1$ is an ultrametric, whereas $s_2$ is not.  }
\label{fig:ultrametric_similarity_nonstar_tree}
 \end{figure}

\subsection{Related Axiomatic, Optimization, and Probabilistic Perspectives}
\label{subsec:related_axiomatic}

A rich line of research has sought to characterize hierarchical clustering procedures through axioms or optimization principles. Our validity-based perspective is closely connected to these efforts yet differs in scope: rather than prescribing a specific algorithm, it formalizes what it means for a hierarchy to be compatible (or valid) with a given similarity function.

\paragraph{Axiomatic characterizations.}
The seminal work of \cite{carlsson2010characterization} proposed a set of axioms for hierarchical clustering viewed as a map from metric spaces to ultrametric spaces.
Imposing invariance under isometries and stability with respect to perturbations of the input metric, they showed that the only admissible method satisfying their axioms is \emph{single linkage}. Their setting, however, treats hierarchical clustering as a procedure producing an ultrametric rather than a hierarchy, and it merges all pairs achieving the current maximal similarity simultaneously, yielding possibly non-binary trees.
In contrast, our framework provides an intrinsic definition of when a given tree is consistent with the similarities, without assuming any particular hierarchical clustering algorithm. 

More recently, \cite{arias2025axiomatic} introduced an axiomatic view for density-based hierarchical clustering, focusing on the stability of density level sets. While their approach also aims at defining hierarchies independently of specific algorithms, it operates in the continuous domain of
measure-based cluster trees. Our notion of validity plays a similar structural role in the discrete similarity setting, identifying the family of hierarchies that can be supported by a given pairwise function.

\paragraph{Objective-based formulations.}
An alternative perspective, initiated by \cite{dasgupta2016cost} and further developed by \cite{cohen2019hierarchical}, formulates hierarchical clustering as an optimization problem over all possible trees. 
In that line of work, the goal is to find a hierarchy minimizing (or maximizing) a global cost functional depending on the pairwise similarities.
However, the optimal tree minimizing the cost function is always binary. 
In contrast, our approach is purely \emph{feasibility-based}: a hierarchy is not optimized but required to satisfy a local set of validity inequalities.
The resulting family of valid hierarchies could thus be viewed as the feasible region underlying those global optimization formulations. 

\paragraph{Relation to our analysis.}
In summary, previous axiomatic and objective formulations identify either a particular algorithm (e.g., single linkage) or a preferred hierarchy optimizing a cost. 
Our contribution complements these perspectives by providing a foundational notion of \emph{validity} that characterizes the entire space of hierarchies consistent with the similarity function. This structural viewpoint isolates the combinatorial core of hierarchical consistency, and it can accommodate both algorithmic and axiomatic specializations as particular instances.

\paragraph{Probabilistic tree inference and network hierarchy models.}
Beyond axiomatic and cost-based frameworks, a significant body of work treats hierarchical clustering as the inference of tree structures under probabilistic or model-selection criteria. For instance, \cite{blundell2010bayesian} propose the Bayesian Rose Trees approach, a greedy‐agglomerative algorithm exploring a posterior over trees with arbitrary branching. Extensions of this idea include non‐parametric priors over hierarchies such as the tree-structured stick-breaking process \citep{ghahramani2010tree} and coalescent-based Bayesian clustering models \citep{teh2007bayesian}. 
In the network domain, recent work considers hierarchical community detection in random graphs via linkage or recursive partitioning methods \citep{li2022hierarchical,dreveton2023does,gray2023hierarchical}. 
These probabilistic and network-driven methodologies complement our validity‐based framework: whereas they typically select or infer a single best tree under a likelihood or prior, our approach focuses on characterizing the entire feasible set of hierarchies consistent with a given similarity function. In this sense, our notion of validity may serve as a constraint or prior domain for tree-inference methods, suggesting a promising direction for future integration between combinatorial validity and probabilistic modeling.

\clearpage

%% file: appendix.tex
\clearpage

\section{Proof of Lemma~\ref{lemma:linkage_contain_iff}}
\label{appendix:proof_trimmed_linkage_recovers}

\subsection{Proof of Lemma~\ref{lemma:linkage_contain_iff}: Sufficiency of Conditions~\ref{condition:linkage_contain_1} and~\ref{condition:linkage_contain_2}}
\label{subsec:proof_sufficiency}

Let $f$ be a linkage update rule satisfying Conditions~\ref{condition:linkage_contain_1} and~\ref{condition:linkage_contain_2}. Let $s$ be an arbitrary similarity function over a finite set $\cX$. Denote by $\Tlinkage = \Tlinkage(\cX,s,f)$ the hierarchy returned by Algorithm~\ref{algo:linkage}, which is a binary tree, and let $T \in \setOfValidHierarchy(\cX,s)$ be a valid hierarchy with respect to $(\cX,s)$. We prove that $T \subseteq \Tlinkage$. 

\paragraph{Step 1. Setting and notation}
Throughout the proof, we use the notations defined in Algorithm~\ref{algo:linkage}, in particular for the sets $\cC^{(m)}_{active}$ that result from the successive aggregations of vertices. In particular, for each iteration $m \in [k-1]$, we let $\cC^{(m)}_{active}$ be the set of active clusters maintained by Algorithm 1. 
We also denote by $\mborn(v)$ the creation time of cluster $v \in \Tlinkage$, 
\textit{i.e.,} $\mborn(v) = m$ if and only if $v \in \acSet{m}$ and $v\notin \acSet{m-1}$.

Initially, $\cC^{(0)}_{active} = \{ \{x_1\}, \cdots, \{x_k\} \}$. Moreover, at all times $m$, we have $\bigcup_{y \in C^{(m)}_{active}} y \weq \cX$ and 
\begin{align}
\label{eq:active_clusters_disjoints}
 y \cap y' \weq \emptyset \quad \forall y \ne y' \in C^{(m)}_{active}.
\end{align}
 Finally, for any $ m \in [k-1]$, we define the property $\cP(m)$: 
\begin{align}
\forall t_1, t_2, t_3 \in \cC_{active}^{(m)}, \text{ if } \quad \exists t \in T \colon t_1, t_2 \subsetneq t, t_3 \subseteq \cX \backslash t \Rightarrow s(t_1,t_2) > \max\left(s(t_1,t_3), s(t_2,t_3) \right). 
\label{eq:cP(m)}
\end{align}
Recall that $T$ is a valid hierarchy. Hence, if $t_1, t_2$ and $t_3$ are three leaves such that $t_1, t_2 \subsetneq t$ and $t_3 \subseteq \cX \backslash t$, then the condition $s(t_1,t_2) > \max\left(s(t_1,t_3), s(t_2,t_3) \right)$ is automatically verified because of (\ref{eq:hierarchical_tree}), independently of the linkage function. But, when $t_1$, $t_2$ and~$t_3$ are not all leaves, the condition $s(t_1,t_2) > \max\left(s(t_1,t_3), s(t_2,t_3) \right)$ is not guaranteed anymore. Indeed, the similarity between non-leaves elements depends on the linkage function, and only assumptions on the linkage function can ensure that $s(t_1,t_2) > \max\left(s(t_1,t_3), s(t_2,t_3) \right)$.

\paragraph{Step 2. If $\cP(m)$ holds for every $m \in [k-1]$, then $T \subseteq \Tlinkage$.}
 Assume that $\cP(m)$ holds for every $m \in [k-1]$. We will show that $T \subseteq \Tlinkage$.

 By contradiction, assume that $T \not\subseteq \Tlinkage$. This is equivalent to the existence of a cluster $t \in T$ such that $t \not\in \Tlinkage$. For this cluster $t$, define
 \begin{align*}
  \cB_t \weq \{ v \in \Tlinkage \colon v \cap t \not\in \{\emptyset, v, t\} \}.
 \end{align*}
  Because $\Tlinkage$ is a binary tree and $t \not\in \Tlinkage$, Lemma~\ref{lemma:v_in_T_bin_if_u_in} asserts that $\cB_t$ \text{ is non-empty.}  
 
 Because $\cB_t$ is non-empty, we can define $u = \argmin_{v \in \cB_t} \mborn(v)$, and let $m' = \mborn(u)$ be the iteration in which the first element of $\cB_t$ is created. 

Because $C^{(0)}_{active}$ is composed of the leaves (which are the singletons $\{x_i\}$), for any $v \in C^{(0)}_{active}$ we have $v\cap t \in \{\emptyset, v \} \subseteq \{\emptyset, v, t \}$. Hence $|u|\ge 2$ and $u$ is created at an iteration $m' \ge 1$ by merging two clusters~$u_1$ and $u_2$, defined by 
 \begin{align}
  \label{in_proof_def_u1u2}
  \{u_1,u_2\} \ \in \ \argmax_{u'_1\ne u'_2 \in \acSet{m'-1}} s(u'_1, u'_2).
 \end{align}
Moreover, because of the definition of $u$, we have $u_1,u_2 \not\in \cB_t$, otherwise, $u \neq \argmin_{v \in \cB_t} \mborn(v)$, hence $u_i \cap t \in \{ \emptyset, u_i, t\}$ ($i\in \{1,2\}$).
Therefore, we have one of the following five cases: 
\begin{enumerate}[(i)]
     \item $u_1 \cap t = t$ or $u_2 \cap t = t$;
     \item $u_1 \cap t = \emptyset$ and $u_2 \cap t = \emptyset$;
     \item $u_1 \cap t = u_1$ and $u_2 \cap t = \emptyset$;
     \item $u_1 \cap t = \emptyset$ and $u_2 \cap t = u_2$;
     \item $u_1 \cap t = u_1$ and $u_2 \cap t = u_2$.
 \end{enumerate}

First, we can rule out Case (i). Indeed, if $u_i \cap t = t$ for some $i \in \{1,2\}$, then we would have $t\cap u = t \cap (u_1 \cup u_2) = t$, which is impossible because $u\in \cB_t$. 

Case (ii) implies that $(u_1 \cup u_2) \cap t = \emptyset$. However, this is again impossible; indeed, $u = u_1 \cup u_2 \in \cB_t$ implies that $(u_1 \cup u_2) \cap t \notin \{\emptyset, u_1 \cup u_2, t\}$. 
Similarly, case (v) is also impossible because it would imply $u = u_1 \cup u_2 \subseteq t$, again an impossibility due to $u \in \cB_t$. 
  As a result, only cases (iii) and (iv) are possible. 

Consider case (iii) where $u_1 \cap t = u_1$ and $u_2 \cap t = \emptyset$, and notice that the symmetric case~(iv) can be treated similarly. 
Because case (i) is impossible, $u_1 \neq t$.  
Together with $\bigcup_{y \in C^{(m'-1)}_{active}} y = \cX$, it therefore implies the existence of a set $u_1' \in C^{(m'-1)}_{active}$ such that $u_1' \ne u_1$ and $u_1' \cap t \ne \emptyset$.
 %

Because $\mborn(u_1') = m'-1 < m' = \mborn(u)$ and because of the definition of $u$, we have $u_1'\notin \cB_t$. Combined with $u_1' \cap t \ne \emptyset$, it implies that $u_1' \cap t \in \{u_1', t\}$. Moreover, because $u_1, u_1' \in \acSet{m'-1}$ are disjoint by~\eqref{eq:active_clusters_disjoints}, we have $(t\cap u_1) \cap u_1' = \emptyset$. Combined with $u_1\cap t = u_1 \neq \emptyset$, this rules out $u_1' \cap t = t$ and forces $u_1' \cap t = u_1'$. 

 To summarize, we have $u_1, u_1', u_2 \in C^{(m'-1)}_{active}$ such that $u_1, u_1' \subsetneq t$ and $u_2 \subseteq \cX \setminus t$. Therefore, the condition $\cP(m'-1)$ implies that $s(u_1, u_1') > s(u_1, u_2)$. This contradicts~\eqref{in_proof_def_u1u2} and therefore proves that $T \subseteq \Tlinkage$. 

\paragraph{Step 3. Inductive proof that  $\cP(m)$ holds for every $m \in [k-1]$.}
We proceed by induction on $m$. 

 Let $m=0$. We have $\cC_{active}^{(0)} = \{\{x_1\}, \cdots, \{x_k\} \}$. Let $t_1, t_2, t_3 \in \cC_{active}^{(0)}$ and suppose there exists $t \in T$ such that $t_1, t_2 \subsetneq t$ and $t_3 \subseteq \cX \backslash t$. 
 Note that $\cC_{active}^{(0)} = \cX$ and thus $\cP(0)$ corresponds to the following statement:
 \begin{align*}
  \forall x_1, x_2, x_3 \in \cX \colon \exists t \in T \colon x_1, x_2 \in t, x_3 \in \cX \backslash t \Rightarrow s\left(x_1, x_2 \right) > \max \left( s \left( x_1, x_3 \right), s \left( x_2, x_3 \right) \right), 
 \end{align*}
 which we can rewrite as
 \begin{align}
 \label{eq:in_proof_P0}
  \forall t \in T \colon \forall x_1, x_2 \in t, x_3 \in \cX \backslash t \Rightarrow s\left(x_1, x_2 \right) > \max \left( s \left( x_1, x_3 \right), s \left( x_2, x_3 \right) \right). 
 \end{align}
 Because $T$ is a valid hierarchy, the statement~\eqref{eq:in_proof_P0} holds. This establishes $\cP(0)$. 
 
 Let $m \ge 0$ and suppose that $\cP(m')$ holds for any $m' \in \{0, \cdots, m\}$. We prove that $\cP(m+1)$ holds. Let $\{t_1', t_2'\} \in \argmax_{v_1 \neq v_2 \in \cC_{active}^{(m)} } s(v_1,v_2)$ and denote $t_{new} = t_1' \cup t_2'$. Recall that, by definition, $\cC_{active}^{(m+1)} \cup \{t_1', t_2' \}= \cC_{active}^{(m)} \cup \{ t_{new} \}$. We consider vertex $t\in T$ and a triplet of three different sets $t_1, t_2, t_3 \in \cC_{active}^{(m+1)}$ such that $t_1, t_2 \subseteq t$ and $t_3 \subseteq \cX \backslash t$. 
 Because $ \cC_{active}^{(m+1)} \setminus \cC_{active}^{(m)} = \{t_{new}\}$, we have the following four cases to consider.
 \begin{itemize}
  \item If $t_1, t_2, t_3 \in \cC_{active}^{(m)}$, then the induction property $\cP(m)$ implies that 
  $$s(t_1,t_2) > \max\left(s(t_1,t_3), s(t_2,t_3) \right).$$ 
  \item Suppose that $t_1 = t_{new} = t_1' \cup t_2'$. Because $t_1, t_2 \subseteq t$, we have $ t_1', t_2', t_2 \subseteq t$. Moreover, $t_1', t_2' \in \cC^{(m)}_{active}$, hence the induction property $\cP(m)$ implies that 
  \begin{align*}
    \forall i \in \{1,2\} \colon \quad s(t_i', t_2) > \max\left(s(t_i',t_3),  s(t_2,t_3) \right).
  \end{align*}
 Thus, Condition~\ref{condition:linkage_contain_1} implies that
 \begin{align*}
  s(t_1' \cup t_2', t_2) > \max\left(s(t_1' \cup t_2', t_3), s(t_2,t_3)\right).
 \end{align*}
 Because $t_1 = t_1' \cup t_2'$, this is equivalent to $s(t_1, t_2) > \max\left(s(t_1, t_3), s(t_2,t_3)\right)$. 
 \item The case $t_2 = t_{new}$ can be treated similarly to the previous case, and leads again to $s(t_1, t_2) > \max\left(s(t_1, t_3), s(t_2,t_3)\right)$. 
 \item Finally, suppose that $t_3 = t_{new}$. Because $t_3 \subseteq \cX \backslash t$, we also have $t_1' \subseteq \cX \backslash t$. Moreover, because $t_1' \in \cC^{(m)}_{active}$, the induction property $\cP(m)$ implies that  
 \begin{align*}
  s(t_1, t_2) > \max\left(s(t_1, t_1'), s(t_2, t_1') \right),
 \end{align*}
 Because we also have $t_2' \subseteq \cX \backslash t$ and $t_2' \in \cC^{(m)}_{active}$, we establish similarly that
 \begin{align*}
   s(t_1, t_2) > \max\left(s(t_1, t_2'), s(t_2, t_2') \right).
 \end{align*}
 Hence, by combining the last two inequalities, we have 
 \begin{align*}
   s(t_1, t_2) > \max\left(s(t_1, t_1'), s(t_2, t_1'), s(t_1, t_2'), s(t_2, t_2') \right). 
 \end{align*} 
 Condition~\ref{condition:linkage_contain_2} therefore implies that $s(t_1, t_2) > \max\left( s( t_1, t_1' \cup t_2' ), s( t_2, t_1' \cup t_2' ) \right)$. Because $t_3 = t_1' \cup t_2'$, this is equivalent to $s(t_1, t_2) > \max\left(s(t_1, t_3), s(t_2,t_3)\right)$. 
\end{itemize}
The combination of these three cases proves $\cP(m+1)$. 

\subsection{Proof of Lemma~\ref{lemma:linkage_contain_iff}: Necessity of Conditions~\ref{condition:linkage_contain_1} and~\ref{condition:linkage_contain_2}}
\label{subsec:proof_necessity}

We call an update rule $f$ \textit{valid} if it satisfies
\begin{align}
\label{eq:def_update_valid}
 \forall (\cX, s) \colon \quad \trueHierarchy(\cX, s) \subseteq \Tlinkage(\cX, s, f). 
\end{align}
We will show that every valid update rule $f$ must satisfy Conditions~\ref{condition:linkage_contain_1} and~\ref{condition:linkage_contain_2}. By contraposition, it suffices to prove the following statement:
\textit{Any update rule $f$ that violates Condition~\ref{condition:linkage_contain_1} or Condition~\ref{condition:linkage_contain_2} is not valid.}

\paragraph{Case (i): Violation of Condition~\ref{condition:linkage_contain_1}}
A violation of Condition~\ref{condition:linkage_contain_1} means that there exist four disjoint clusters $t_1, \cdots, t_4$ such that $t_1$ and $t_2$ are merged at some step, but
 \begin{align}
 \label{eq:violation_condition_1_a}
  \forall i \in \{1,2\} \colon s(t_i,t_3) \ > \ \max \left(s(t_i,t_4),s(t_3,t_4)\right)
 \end{align}
 and
 \begin{align}
  \label{eq:violation_condition_1_b}
  f(s_{12}, s_{23}, s_{31}, n_1,n_2,n_3) \wle \max\left(f(s_{12}, s_{24}, s_{41}, n_1,n_2,n_4), s_{34} \right), 
 \end{align}
where we use the shorthand $s_{ij} = s(t_i,t_j)$ and $n_i = |t_i|$. Because $t_1$ and $t_2$ are the clusters merged by the algorithm, we also have $s(t_1,t_2) \in \argmax_{i \in [4]} s(t_i,t_j)$ and thus
\begin{align}
\label{eq:t12_picked}
s(t_1,t_2) \ge s(t_i,t_j) \quad \text{for all } i,j \in [4], i\neq j.
\end{align}

\textit{Construction of a counterexample.} 
Suppose $f$ violates Condition~\ref{condition:linkage_contain_1} but is valid. We will construct a dataset $(\cX_1,s_1)$ for which
$\trueHierarchy(\cX_1,s_1) \not \subseteq \Tlinkage(\cX_1,s_1,f)$,
contradicting \eqref{eq:def_update_valid} and thus the validity of $f$.

Let $t_1', t_2',t_3', t_4'$ be disjoint sets with $|t_i'| = n_i$, and define $\cX_1 = \bigcup_{i\in [4]} t_i'$. 
Define the similarity function $s_1 \colon \cX_1 \times \cX_1 \to \R_+$ by 
\begin{align*}
s_1(x,y) \weq
\begin{cases}
s_{ij}, & \text{if } x \in t_i',\; y \in t_j' \text{ for some } i \ne j \in \{1,2,3,4\}, \\
s_{12}+1, & \text{if } x,y \in t_i' \text{ for some } i \in \{1,2,3,4\}.
\end{cases}
\end{align*}
By~\eqref{eq:violation_condition_1_a} and~\eqref{eq:t12_picked},
for each $i \in \{1,2\}$, we have 
\[
s_{12} \ge s_{i3} > \max(s_{i4}, s_{34}).\]
Hence, the union $t_1' \cup t_2' \cup t_3'$ forms a valid cluster with respect to $s_1$, \textit{i.e.,} 
\[
 t_1' \cup t_2' \cup t_3' \in \trueHierarchy(\cX_1,s_1). 
\]
By Lemma~\ref{lemma:ultrametric}, the clusters $t_1', \cdots, t_4'$ appear as active clusters before any cross-merging occurs, so for all $1 \le i < j \le 4$,
\begin{align*}
 s_1(t_i',t_j') \weq s_{ij}.
\end{align*}
Moreover, since $t_1'$ and $t_2'$ are merged first (by~\eqref{eq:t12_picked}), the similarity between the cluster $t_1'\cup t_2'$ formed by their merger and $t_3'$ satisfies 
\begin{align}
    s_1(t_1'\cup t_2', t_3') 
    & \weq f(s_{12}, s_{23}, s_{31}, n_1,n_2,n_3) \nonumber \\
    & \wle \max\left(f(s_{12}, s_{24}, s_{41}, n_1,n_2,n_4), s_{34} \right) \label{eq:ineq_in_proof} \\
    & \weq \max\left(s_1(t_1' \cup t_2', t_4'), s_1(t_3',t_4') \right), \nonumber
\end{align}
where \eqref{eq:ineq_in_proof} follows from \eqref{eq:violation_condition_1_b}. 
If the inequality~\eqref{eq:ineq_in_proof} is strict, then the algorithm at the next step merges either the pair $(t_1' \cup t_2',t_4')$ or the pair $(t_3',t_4')$, preventing the formation of the true cluster $t_1'\cup t_2' \cup t_3'$. 
Even if equality holds, the similarity of the pair $(t_1'\cup t_2', t_3')$ merely ties with the similarity of another pair, and tie-breaking may again exclude the valid cluster.

In both cases,
\[
\trueHierarchy(\cX_1, s_1) \not\subseteq \Tlinkage(\cX_1, s_1,f),
\]
contradicting \eqref{eq:def_update_valid}. Thus, a valid $f$ cannot violate Condition~\ref{condition:linkage_contain_1}.

\paragraph{Case (ii): Violation of Condition~\ref{condition:linkage_contain_2}}

A violation of Condition~\ref{condition:linkage_contain_2} means that there exist four disjoint clusters $t_1, t_2,t_3, t_4$ such that $t_3$ and $t_4$ are merged, i.e.
\begin{align}
\label{eq:t34_picked}
s(t_3,t_4) \ge s(t_i,t_j) \quad \text{for all } i,j \in [4], i\neq j,
\end{align}
and simultaneously
\begin{align}
\label{eq:violation_condition_2}
s_{12} > \max_{\substack{i \in \{1,2\}, \, j \in \{3,4\}}} s_{ij}
\quad \text{and} \quad
s_{12} \le \max_{i \in \{1,2\}} f(s_{34}, s_{4i}, s_{3i}, n_3,n_4,n_i).
\end{align}

\textit{Construction of a counterexample.} Assume again that $f$ is valid but violates Condition~\ref{condition:linkage_contain_2}.

Let $t_1', t_2',t_3', t_4'$ be disjoint sets with $|t_i'| = n_i$, and define $\cX_1 = \bigcup_{i\in [4]} t_i'$. 
Define the similarity function $s_1 \colon \cX_1 \times \cX_1 \to \R_+$ by 
\begin{align*}
s_1(x,y) \weq
\begin{cases}
s_{ij}, & \text{if } x \in t_i',\; y \in t_j' \text{ for some } i \ne j \in \{1,2,3,4\}, \\
s_{34}+1, & \text{if } x,y \in t_i' \text{ for some } i \in \{1,2,3,4\}.
\end{cases}
\end{align*}

From~\eqref{eq:violation_condition_2}, $s_{12}$ is strictly larger than $\max_{ \substack{i \in \{1,2\}, \, j \in \{3,4 \} } } s_{ij}$, hence 
\[
t_1' \cup t_2' \in \trueHierarchy(\cX_1, s_1). 
\]
By Lemma~\ref{lemma:ultrametric}, $t_1',\cdots,t_4'$ appear as active clusters before any inter-cluster merge, and 
\[
s_1(t_i',t_j') \weq s_{ij} \quad \forall i < j.
\]
By~\eqref{eq:violation_condition_2}, we have
\begin{align}
    s_1(t_1', t_2') & \weq s_{12} \notag \\
    &\wle \max_{i \in \{1,2\} }f(s_{34}, s_{4i}, s_{3i}, n_3,n_4,n_i) \\
    &\weq \max_{i \in \{1,2\} }s_1(t_3'\cup t_4', t_i'). \notag
\end{align}

Therefore, the linkage procedure either merges one of the spurious pairs $(t_3' \cup t_4', t_i')$ or ties with the valid pair $(t_1', t_2')$, preventing the formation of the true cluster $t_1' \cup t_2'$. 
Hence
\[
\trueHierarchy(\cX_1, s_1) \not\subseteq \Tlinkage(\cX_1, s_1,f),
\]
contradicting \eqref{eq:def_update_valid}. Therefore a valid $f$ cannot violate Condition~\ref{condition:linkage_contain_2}.

\begin{remark}
Without Assumption~\ref{assumption:f_r_cont}, we can only show $f(q^+,q,q,\cdots) = q$ for any $q^+ > q$ (Lemma~\ref{lemma:ultrametric}) and $f(q,q,q,\cdots) = q$ might not hold. This forbids us from using the ultrametric similarities (\textit{i.e.,} $s(x,y)$ that takes the same value for any $x,y \in t_i'$). However, careful constructions of $(\cX_1,s_1)$ that serve as the counterexamples that are similar to those in the above proof (but more complex), are still possible but require more tedious arguments.
\end{remark}

\subsection{Auxiliary Lemmas}
\label{subsec:auxiliary_lemmas_deterministic_results}
\subsubsection{Auxiliary Lemmas for Appendix~\ref{subsec:proof_sufficiency}}
The following technical lemma is used to prove the sufficiency of Conditions~\ref{condition:linkage_contain_1} and~\ref{condition:linkage_contain_2} in Lemma~\ref{lemma:linkage_contain_iff}. 

\begin{lemma}\label{lemma:v_in_T_bin_if_u_in}
    Let $v \subseteq \cX$. Let $T_{bin} \in \cT(\cX)$ be a binary tree. The following holds: 
%
    \begin{align}
    \label{condition:lemma_binary_tree}
    v \not\in T_{bin} \quad \Rightarrow \quad \exists u \in T_{bin} \colon u \cap v \not\in \{\emptyset, u, v\}. 
    \end{align}
\end{lemma}
 
\begin{proof} Let $v \subseteq \cX$ such that $v \not\in T_{bin}$. 
We show that there exists $u_* \in T_{bin}$ such that $u_* \cap v \not\in \{\emptyset, u_*, v\}$.

Because $ v \notin T_{bin}$ while $\cX \in T_{bin}$, we have $v \subsetneq \cX$.
 Denote by $\cB$ the collection of ancestors of $v$ in the tree
 \begin{align*}
     \cB \weq \{ u \in T_{bin} \colon v \subsetneq u\}.
 \end{align*}
 Because $\cX \in \cB$, the set $\cB$ is not empty. Let
 \begin{align*}
     u_0 \weq \argmin \{ |u| \colon u \in \cB \}
 \end{align*}
 be the smallest cluster of $T_{bin}$ strictly containing $v$. 
 
 Because $v \subsetneq u_0$, $u_0$ is not a leaf. Because $T_{bin}$ is a binary tree, $u_0$ has exactly two children~$u_1$ and $u_2 \in T_{bin}$ satisfying $u_1 \cup u_2 = u_0$ and $u_1 \cap u_2 = \emptyset$. 
 
 By minimality of $u_0$, neither child fully contains $v$; otherwise that child would belong to~$\cB$ with smaller cardinality that $u_0$. 
 Hence,
 \begin{align*}
    v \not \subseteq u_1
    \quad \text{ and } \quad
    v \not \subseteq u_2.
 \end{align*}
Moreover, at least one child is not contained in $v$, \textit{i.e.,}
\begin{align*}
    u_1 \not\subseteq v \text{ or } u_2 \not\subseteq v. 
\end{align*}
Indeed, by contradiction, suppose that both $u_1 \subseteq v$ and $u_2 \subseteq v$ hold. Then, $u_0 = u_1 \cup u_2 \subseteq v$, contradicting $v \subsetneq u_0$. 
    
 Without loss of generality, suppose that $u_1 \not \subseteq v$, and let $u_* = u_1$. We have $u_* \in T_{bin}$. Because $u^*$ neither contains $v$ nor is disjoint from it, we have 
 \begin{align*}
u_* \cap v \ \not\in \ \{\emptyset, u_*, v\}.
 \end{align*}
 This ends the proof.
\end{proof}

\subsubsection{Auxiliary Lemmas for Appendix~\ref{subsec:proof_necessity}}
\label{subsec:auxiliary_proof_necesity}

 We first introduce the following notations, used throughout this section.
\paragraph{Notations.}
\begin{itemize}
    \item For any subset $\cX_1 \subseteq \cX$, denote by $s[\cX_1]$ the restriction of the similarity function $s$ to~$\cX_1$, that is, 
    \[
    s[\cX_1] \colon (x,y) \in \cX_1 \times \cX_1 \mapsto s(x,y).
    \]
    \item Denote by $\acSet{m}(\cX,s)$ the active set at iteration $m$ of the linkage process applied to $(\cX,s)$. When unambiguous, we simply write $\acSet{m}$.
\end{itemize}

We present two auxiliary lemmas used in Appendix~\ref{subsec:proof_necessity}.

\begin{lemma}[Bounds on updated similarities]
Let $f$ be a valid update function, and let $t_1, t_2, t_3, t_4, t_5$ be disjoint clusters.  
Then the following statements hold for the linkage process defined by $f$.
\begin{enumerate}
    \item Suppose there exists $(\cX,s)$ such that $t_1, t_2, t_3, t_4 \subsetneq \cX$, the union $t_1 \cup t_2 \cup t_3$ is a valid cluster, and for some iteration $m$,
    \[
    t_1 \cup t_2, \; t_3, \; t_4 \in \acSet{m}(\cX,s).
    \]
    Then necessarily,
    \[
    s(t_1 \cup t_2, t_3) > s(t_3,t_4).
    \]
    \item Suppose there exists $(\cX,s)$ such that $t_1, t_2, t_3, t_5 \subsetneq \cX$, the union $t_3 \cup t_5$ is a valid cluster, and for some iteration $m$,
    \[
    t_1 \cup t_2, \; t_3, \; t_5 \in \acSet{m}(\cX,s).
    \]
    Then necessarily,
    \[
    s(t_1 \cup t_2, t_3) < s(t_3,t_5).
    \]
    \item Combining the previous two statements, if there exists $(\cX,s)$ such that 
    $(t_1 \cup t_2 \cup t_3 \cup t_4, s[t_1 \cup t_2 \cup t_3 \cup t_4])$ satisfies the assumptions of point~1 and 
    $(t_1 \cup t_2 \cup t_3 \cup t_5, s[t_1 \cup t_2 \cup t_3 \cup t_5])$ satisfies those of point~2, then
    \[
    s(t_3,t_4) < s(t_1 \cup t_2, t_3) < s(t_3,t_5).
    \]
\end{enumerate}
\label{lemma:bound_updated_sim}
\end{lemma}

\begin{proof}
\begin{enumerate}
 \item Because $t_1 \cup t_2 \cup t_3$ is a valid cluster, the algorithm must merge $t_3$ with $t_1 \cup t_2$ before any merge involving $t_4$. If $s(t_1 \cup t_2, t_3) \le s(t_3,t_4)$, the merge between $t_3$ and $t_4$ could occur first, violating the validity of $t_1 \cup t_2 \cup t_3$. Hence, $s(t_1 \cup t_2, t_3) > s(t_3,t_4)$.
 \item Because $t_3 \cup t_5$ is a valid cluster, the linkage must merge $t_3$ and $t_5$ before any merge involving $t_1 \cup t_2$. If $s(t_1 \cup t_2, t_3) \ge s(t_3,t_5)$, the pair $(t_1 \cup t_2, t_3)$ could merge first, violating the validity of $t_3 \cup t_5$. Thus, $s(t_1 \cup t_2, t_3) < s(t_3,t_5)$.
\item Follows directly from combining the inequalities in points 1 and 2.
\end{enumerate}
\end{proof}

The following lemma formalizes a key property of any valid update function $f$: when both its second and third arguments are equal to $q$, the function must output $q$, i.e., $f(q^+, q, q, \cdots) = q$ for all $q^+ \ge q$. 

\begin{lemma}
Let $f$ be a valid update function.  
Consider any $(\cX,s)$ and disjoint clusters $t_1,t_2$ such that $s(x,y)=q$ for all $x\in t_1, y\in t_2$.  
If there exists an iteration $m$ such that $t_1,t_2\in \acSet{m}(\cX,s)$, then
\[
s(t_1,t_2) = q.
\]
Moreover, for any $q^+ \ge q$ and any $n_1,n_2,n_3 \in \N$,
\[
f(q^+, q, q, n_1, n_2, n_3) = q.
\]
\label{lemma:ultrametric}
\end{lemma}

\begin{proof}
We prove the claim by induction on the cluster sizes $n_1,n_2$.

\paragraph{Claim.}
For any $n_1,n_2$ and any $q\in\R$, there exists $m$ such that if $t_1,t_2\in\acSet{m}$ with $|t_1|=n_1$, $|t_2|=n_2$, and $s(x,y)=q$ for all $x\in t_1,y\in t_2$, then $s(t_1,t_2)=q$.

\paragraph{Base case ($n_1=n_2=1$).}  
Trivial, since the similarity between singletons equals $s(x,y)=q$ by definition.

\paragraph{Inductive hypothesis.}  
Assume the claim holds for all $n_1\le b_1$ and $n_2\le b_2$.  
Without loss of generality, assume $b_1\ge b_2$ (since $s$ is symmetric).

\paragraph{Inductive step.}  
We prove the claim for all $(n_1,n_2)$ with 
$n_1 \le \max(2b_1,b_2)$ and $n_2 \le \min(2b_1,b_2)$.

Let $t_1,t_2,t_3,t_4,t_5$ be disjoint clusters such that 
$|t_1|,|t_2|,|t_4|,|t_5|\le b_2$ and $|t_3|\le b_1$.  
Fix arbitrary real numbers $q_0 > q^+ > q+\delta$ for some $\delta>0$.  
Define $\cX = t_1 \cup t_2 \cup t_3 \cup t_4 \cup t_5$ and a similarity function $s:\cX\times\cX\to\R$ by
\[
s(x,y) \weq 
\begin{cases}
    q_0, & (x,y)\in t_i\times t_i \text{ for some } i\in[5],\\
    q^+, & (x,y)\in t_1\times t_2,\\
    q,   & (x,y)\in (t_1\cup t_2)\times (t_3\cup t_4),\\
    q+\delta, & (x,y)\in t_3\times t_4,\\
    q-\delta, & (x,y)\in (t_1\cup t_2\cup t_3\cup t_4)\times t_5.
\end{cases}
\]
Observe that all $t_i$ are valid clusters with respect to $s$. By the inductive hypothesis, at every step of the linkage process:
\[
s(u,v) \weq 
\begin{cases}
    q_0, & u,v \subsetneq t_i,\\
    q^+, & u\subseteq t_1, v\subseteq t_2,\\
    q, & u\subsetneq (t_1\cup t_2),\ v\subsetneq (t_3\cup t_4),\\
    q+\delta, & u,v \subsetneq t_3\times t_4,\\
    q-\delta, & u\subsetneq (t_1\cup t_2\cup t_3\cup t_4),\ v\subseteq t_5.
\end{cases}
\]

Hence, there exist $m_1,m_2$ such that
\[
t_1,t_2,t_3,t_4 \in \acSet{m_1}(t_1\cup t_2\cup t_3\cup t_4, s[t_1 \cup t_2 \cup t_3 \cup t_4]),
\]
and
\[
t_1,t_2,t_3,t_5 \in \acSet{m_2}(t_1\cup t_2\cup t_3\cup t_5, s[t_1 \cup t_2 \cup t_3 \cup t_5]).
\]
Moreover, by the construction of $s$, 
\[
s(t_i,t_j) \weq
\begin{cases}
    q^+, & (i,j)=(1,2),\\
    q, & i\in\{1,2\}, j\in\{3,4\},\\
    q+\delta, & (i,j)=(3,4),\\
    q-\delta, & i\in[4], j=5.
\end{cases}
\]
Because $s(t_1,t_2)$ is the unique maximum among all pairwise similarities $s(t_i,t_j)$, we have:
\[
t_1 \cup t_2,t_3,t_4 \in \acSet{m_1+1}(t_1\cup t_2 \cup t_3 \cup t_4,s[t_1\cup t_2 \cup t_3 \cup t_4]),
\] 
and
\[
t_1 \cup t_2,t_3,t_5 \in \acSet{m_2+1}(t_1\cup t_2 \cup t_3 \cup t_5,s[t_1\cup t_2 \cup t_3 \cup t_5]). 
\]
Moreover, $t_1\cup t_2,t_1 \cup t_2 \cup t_3$ are valid clusters with respect to $(t_1\cup t_2 \cup t_3 \cup t_4,s[t_1\cup t_2 \cup t_3 \cup t_4])$, and $t_3 \cup t_5$ is a valid cluster with respect to $(t_1\cup t_2 \cup t_3 \cup t_5,s[t_1\cup t_2 \cup t_3 \cup t_5])$.

Applying Lemma~\ref{lemma:bound_updated_sim} (point 3) yields:
\[
s(t_3,t_4) \weq q-\delta \ < \ s(t_1\cup t_2, t_3) \ < \ s(t_3,t_5) \weq q+\delta.
\]
Because this holds for any $\delta>0$, we conclude that
\[
s(t_1\cup t_2, t_3) \weq q.
\]
Thus the claim holds for all relevant cluster sizes, completing the induction.

Finally, this implies that
\[
f(q^+, q, q, n_1, n_2, n_3) = q
\]
for all $q^+ > q$.  
By the right-continuity of $f$ (Assumption~\ref{assumption:f_r_cont}), the equality also holds for $q^+=q$.
\end{proof}

\section{Additional Proofs}

\subsection{Proof of Lemma~\ref{lemma:standardLinkage_satisfy_conditions}}

We prove Lemma~\ref{lemma:standardLinkage_satisfy_conditions} separately for unweighted and weighted average linkages (Lemma~\ref{lemma:average_satisfies_conditions}), and single and complete linkages (Lemma~\ref{lemma:single_complete_satisfies_conditions}).

\begin{lemma}
\label{lemma:average_satisfies_conditions}
 Weighted and unweighted average linkages satisfy Conditions~\ref{condition:linkage_contain_1} and~\ref{condition:linkage_contain_2}.
\end{lemma}

\begin{proof}[Proof of Lemma~\ref{lemma:average_satisfies_conditions}]
First, we show that weighted and unweighted average linkage satisfy Condition~\ref{condition:linkage_contain_1}. 
Consider vertices $t_1,t_2, t_3,t_4$ such that $\forall i \in \{1,2\},~ s(t_i,t_3)> \max \left(s(t_i,t_4),s(t_3,t_4)\right)$. Denote $\eta = \frac{|t_1|}{|t_1|+|t_2|}$ for weighted average and $\eta=1/2$ for unweighted average. We have
\begin{align*}
    & s(t_1\cup t_2,t_3) - \max\left(s(t_1\cup t_2,t_4) , s(t_3,t_4)\right) \\
    & \weq \eta s(t_1,t_3) + (1-\eta) s(t_2,t_3) - \max\left(\eta s(t_1,t_4) + (1-\eta) s(t_2,t_4), s(t_3,t_4) \right) \\
    & \wge \eta \left(s(t_1,t_3)-\max\left( s(t_1,t_4),s(t_3,t_4) \right)\right) + (1-\eta) \left(s(t_2,t_3)-\max\left(s(t_2,t_4),s(t_3,t_4) \right)\right) \\
    & \ > \ 0,
\end{align*}
and thus Condition~\ref{condition:linkage_contain_1} is verified. 

Now, we show that these average linkages also satisfy Condition~\ref{condition:linkage_contain_2}. 
Consider vertices $t_1, t_2, t_3, t_4$ such that $s(t_1,t_2)> \max_{i \in \{1,2 \}, j \in \{3,4\}} s(t_i,t_j)$. 
Then, for these vertices, we have
\begin{align*}
    s(t_1,t_2) - \max_{i \in \{1,2 \}} s(t_i,t_3\cup t_4)
    & \weq \min_{i \in \{1,2 \}} \eta\cdot \left(s(t_1,t_2)- s(t_i, t_3) \right) + (1-\eta)\cdot \left(s(t_1,t_2)-s(t_i,t_4)\right) \notag\\
    & \ > \ 0,
\end{align*}
where $\eta = \frac{|t_3|}{|t_3|+|t_4|}$ for weighted average and $\eta=1-\eta=1/2$ for unweighted average. 
\end{proof}

\begin{lemma}
\label{lemma:single_complete_satisfies_conditions}
 Single and complete linkages satisfy Conditions~\ref{condition:linkage_contain_1} and \ref{condition:linkage_contain_2}.
\end{lemma}

\begin{proof}[Proof of Lemma~\ref{lemma:single_complete_satisfies_conditions}]
We first show that complete and single linkages satisfy Condition~\ref{condition:linkage_contain_1}. Consider vertices $t_1,t_2, t_3,t_4$ such that $\forall i \in \{1,2\},~ s(t_i,t_3)> \max \left(s(t_i,t_4),s(t_3,t_4)\right)$. Moreover, without loss of generality, assume that $s(t_1, t_3) > s(t_2,t_3)$.
Then, with complete linkage, we have
\begin{align*}
    s(t_1\cup t_2,t_3) - s(t_3, t_4)
    & \weq \min\left(s(t_1, t_3), s(t_2, t_3) \right)- s(t_3, t_4) \\
    & \weq s(t_2, t_3) - s(t_3, t_4) > 0,
\end{align*}
and
\begin{align*}
    s(t_1\cup t_2,t_3) - s(t_1\cup t_2,t_4)
    & \weq \min\left(s(t_1, t_3), s(t_2, t_3) \right) - \min\left(s(t_1, t_4), s(t_2, t_4) \right)\notag\\
    & \weq s(t_2, t_3) - \min\left(s(t_1, t_4), s(t_2, t_4) \right) \notag\\
    & \weq \begin{cases}
        s(t_2,t_3) - s(t_1, t_4) \quad\text{if } 
        s(t_1,t_4) \leq s(t_2,t_4) \\
        s(t_2,t_3) - s(t_2, t_4) \quad\text{if } 
        s(t_1,t_4) > s(t_2,t_4)
    \end{cases} \notag\\
    & \wge  s(t_2,t_3) - s(t_2, t_4)
    \ > \ 0.
\end{align*}
Thus, complete linkage satisfies Condition~\ref{condition:linkage_contain_1}.

Similarly, for single linkage, we have
\begin{align*}
    s(t_1\cup t_2,t_3) - s(t_3, t_4)
    & \weq \max\left(s(t_1, t_3), s(t_2, t_3) \right)- s(t_3, t_4) \\
    & \weq s(t_1, t_3) - s(t_3, t_4) > 0,
\end{align*}
and
\begin{align*}
    s(t_1\cup t_2,t_3) - s(t_1\cup t_2,t_4)
    & \weq \max\left(s(t_1, t_3), s(t_2, t_3) \right) - \max\left(s(t_1, t_4), s(t_2, t_4) \right)\notag\\
    & \weq s(t_1, t_3) - \max\left(s(t_1, t_4), s(t_2, t_4) \right) \notag\\
    & \weq \begin{cases}
        s(t_1,t_3) - s(t_2, t_4) \quad\text{if } 
        s(t_1,t_4) \leq s(t_2,t_4) \\
        s(t_1,t_3) - s(t_1, t_4) \quad\text{if } 
        s(t_1,t_4) > s(t_2,t_4)
    \end{cases} \notag\\
    & \wge \begin{cases}
        s(t_2,t_3) - s(t_2, t_4) \quad\text{if } 
        s(t_1,t_4) \leq s(t_2,t_4) \\
        s(t_1,t_3) - s(t_1, t_4) \quad\text{if } 
        s(t_1,t_4) > s(t_2,t_4)
    \end{cases} \notag\\
    & \ > \ 0.
\end{align*}
Hence, single linkage satisfies Condition~\ref{condition:linkage_contain_1}.

We now establish that complete and single linkages also satisfy Condition~\ref{condition:linkage_contain_2}. Consider vertices $t_1, t_2, t_3, t_4$ such that $s(t_1,t_2)> \max_{i \in \{1,2 \}, j \in \{3,4\}} s(t_i,t_j)$. For complete linkage, we have
\begin{align*}
    s(t_1,t_2) - \max_{i \in \{1,2 \}} s(t_i,t_3\cup t_4)
    & \weq s(t_1,t_2) - \max_{i \in \{1,2 \}} \min_{j \in \{3,4 \} }s(t_i,t_j) \notag\\
    & \wge s(t_1,t_2) - \max_{i \in \{1,2 \}} \max_{j \in \{3,4 \} }s(t_i,t_j) \notag\\
    & \ > \ 0.
\end{align*}
Similarly, for single linkage,  we have
\begin{align*}
    s(t_1,t_2) - \max_{i \in \{1,2 \}} s(t_i,t_3\cup t_4)
    & \weq s(t_1,t_2) - \max_{i \in \{1,2 \}} \max_{j \in \{3,4 \} }s(t_i,t_j) \notag\\
    & \ > \ 0.
\end{align*}
Therefore, complete and single linkage also satisfy Condition~\ref{condition:linkage_contain_2}.
\end{proof}

\subsection{Proof of Lemma~\ref{lemma:someLinkage_break_conditions}}
\label{subsec:counterexamples_other_linkages}

We prove that Ward violates Condition~\ref{condition:linkage_contain_1_dissimilarity} in Lemma~\ref{lemma:ward_violates_cond1}, and that Centroid and Median linkages  violate Condition~\ref{condition:linkage_contain_2_dissimilarity} in Lemma~\ref{lemma:PGMCs_violates_cond2}.

\begin{lemma}Ward does not satisfy Condition~\ref{condition:linkage_contain_1_dissimilarity}.
\label{lemma:ward_violates_cond1}
\end{lemma}
\begin{proof} Assume the case where $|t_3| > |t_4|$. Then
\begin{align*}
    \MoveEqLeft d(t_1 \cup t_2,t_4) - d(t_1\cup t_2, t_3) \\
    &= \frac{|t_1|+|t_4|}{|t_1|+|t_2|+|t_4|} d(t_1, t_4) + \frac{|t_2|+|t_4|}{|t_1|+|t_2|+|t_4|} d(t_2, t_4) - \frac{|t_4|}{|t_1|+|t_2|+|t_4|} d(t_1, t_2) \\
    & \quad - \frac{|t_1|+|t_3|}{|t_1|+|t_2|+|t_3|} d(t_1, t_3) - \frac{|t_2|+|t_3|}{|t_1|+|t_2|+|t_3|} d(t_2, t_3) + \frac{|t_3|}{|t_1|+|t_2|+|t_3|} d(t_1, t_2) \\
    &= \frac{|t_1|+|t_4|}{|t_1|+|t_2|+|t_4|} \left(d(t_1, t_4) - d(t_1,t_3)\right) + \frac{|t_2|+|t_4|}{|t_1|+|t_2|+|t_4|} \left(d(t_2, t_4) - d(t_2,t_3)\right) \\
    & \quad - \frac{\left(|t_3|-|t_4|\right)\left(|t_2|\left( d(t_1,t_4) - d(t_1,t_2)\right) + |t_1|\left( d(t_2,t_4) - d(t_1,t_2)\right)\right)}{\left(|t_1|+|t_2|+|t_3|\right)\left(|t_1|+|t_2|+|t_4|\right)}.
\end{align*}
Here, observe that $|t_3| > |t_4|$, $d(t_1,t_4) - d(t_1, t_2) > 0$, and $d(t_1, t_2)$ can be very close to $0$, while $d(t_1, t_3)$ and $d(t_2, t_3)$ can be arbitrarily close to $d(t_1,t_4)$ and $d(t_2,t_4)$, respectively. Therefore $d(t_1 \cup t_2,t_4) - d(t_1\cup t_2, t_3)$ can take non-positive values.
Moreover, 
\begin{align*}
    \MoveEqLeft d(t_3,t_4) - d(t_1\cup t_2, t_3) \\
    &= d(t_3,t_4) - \eta_1 d(t_1, t_3) -\eta_2 d(t_2, t_3) - \beta d(t_1, t_2) \\
    &=  \eta_1 \left(d(t_3,t_4) - d(t_1, t_3) \right) + \eta_2 \left(d(t_3,t_4) - d(t_2, t_3) \right) +  \beta \left(d(t_3,t_4) - d(t_1, t_2) \right),
\end{align*}
where $\eta_1 = \frac{|t_1|+|t_3|}{|t_1|+|t_2|+|t_3|}$, $\eta_2 = \frac{|t_2|+|t_3|}{|t_1|+|t_2|+|t_3|}$, and $\beta=-\frac{|t_3|}{|t_1|+|t_2|+|t_3|}$.
Observe that $\beta < 0$, $d(t_3,t_4) - d(t_1, t_2) > 0$, and $d(t_1, t_2)$ can be arbitrarily close to $0$ while $d(t_1, t_3)$ and $d(t_2, t_3)$ can be arbitrarily close to $d(t_3,t_4)$. Therefore $d(t_3,t_4) - d(t_1\cup t_2, t_3)$ can take non-positive values. Consequently, Ward does not always respect Condition~\ref{condition:linkage_contain_1}.
\end{proof}

\begin{lemma}
Median and Centroid linkages do not satisfy Condition~\ref{condition:linkage_contain_2_dissimilarity}. 
\label{lemma:PGMCs_violates_cond2}
\end{lemma}
\begin{proof}
Without loss of generality, assume $d(t_1, t_3 \cup t_4) < d(t_2, t_3 \cup t_4)$
\begin{align*}
    \min_{i \in \{1,2 \}}\left(t_i, t_3 \cup t_4 \right) - d(t_1, t_2)
    &= d(t_1, t_3 \cup t_4) - d(t_1, t_2) \\
    &= \eta d(t_1, t_3) + (1-\eta) d(t_1, t_4) + \beta d(t_3, t_4) - d(t_1, t_2)\\
    &= \eta \left(d(t_1, t_3)- d(t_1, t_2) \right) + (1-\eta) \left(d(t_1, t_4)- d(t_1, t_2) \right) + \beta d(t_3, t_4),
\end{align*}
where $\eta = \frac{|t_3|}{|t_3|+|t_4|}$ ($1-\eta=\frac{|t_3|}{|t_3|+|t_4|}$), and $\beta=-\frac{|t_3||t_4|}{(|t_3|+|t_4|)^2}$ for Median method, $\eta=1-\eta=1/2$ and $\beta=-1/4$ for Centroid. 
Here, as $\beta < 0$, and $d(t_1,t_3)$ and as $d(t_1,t_4)$ can be very close to $d(t_1, t_2)$, $\min_{i \in \{1,2 \}}\left(t_i, t_3 \cup t_4 \right) - d(t_1, t_2)$ can be non-positive. As a result, Median and Centroid do not always respect Condition~\ref{condition:linkage_contain_2}.
\end{proof}

Finally, Figure~\ref{fig:ward_fails} provides an explicit numeric example where Ward fails to return the finest valid hierarchy. 

\begin{figure}[!ht]
 \begin{subfigure}{0.4\linewidth}
\begin{align*}
 \begin{array}{cc}
 & 
 \begin{array}{ccccc}
 x_1 & x_2 & x_3 & x_4 & x_5
 \end{array} 
 \\
 \begin{array}{c}
 x_1\\
 x_2\\
 x_3\\
 x_4\\
 x_5
 \end{array}
 &
 \begin{pmatrix}
 0 & 4 & 15 & 15 & 24 \\
 4 & 0 & 15 & 15 & 24 \\
 15 & 15 & 0 & 3 & 21 \\
 15 & 15 & 3 & 0 & 21 \\
 24 & 24 & 21 & 21 & 0
 \end{pmatrix}
 \end{array}
 \end{align*}
 \caption{Dissimilarity $d$}
 \label{fig:ward_fails_dissimilarity}
 \end{subfigure}
\begin{subfigure}{0.29\textwidth}
\centering
\begin{tikzpicture}
[level distance=8mm,
   level 1/.style={sibling distance=10mm,nodes={}},
   level 2/.style={sibling distance=10mm,nodes={}},
   level 3/.style={sibling distance=5mm,nodes={}},
]
\coordinate
    child { 
        child {
            child {node {$x_1$}} 
            child {node {$x_2$}}
        } 
        child {
            child {node {$x_3$}} 
            child {node {$x_4$}}
        }
        }
    child { node {$x_5$}
    } 
    ;
\end{tikzpicture}
\caption{$\trueHierarchy(\cX,d)$}
\label{fig:ward_fails_trueHierarchy}
\end{subfigure}
\hfil
\begin{subfigure}{0.29\textwidth}
\centering
\begin{tikzpicture}
[level distance=8mm,
   level 1/.style={sibling distance=14mm,nodes={}},
   level 2/.style={sibling distance=8mm,nodes={}},
   level 3/.style={sibling distance=5mm,nodes={}},
]
\coordinate
    child { 
        child {
            child {node {$x_1$}} 
            child {node {$x_2$}}
        } 
        child {node {$x_5$}}
        }
    child {
            child {node {$x_3$}} 
            child {node {$x_4$}}
    } 
    ;
\end{tikzpicture}
\caption{$T_{ward}(\cX,d)$}
\label{fig:ward_fails_tree_ward}
\end{subfigure}
\caption{Figure~\ref{fig:ward_fails_dissimilarity} gives a dissimilarity function over a set of 5 items. The valid hierarchy is provided in Figure~\ref{fig:ward_fails_trueHierarchy}, and Figure~\ref{fig:ward_fails_tree_ward} shows the tree constructed by Ward linkage. }
\label{fig:ward_fails}
\end{figure}
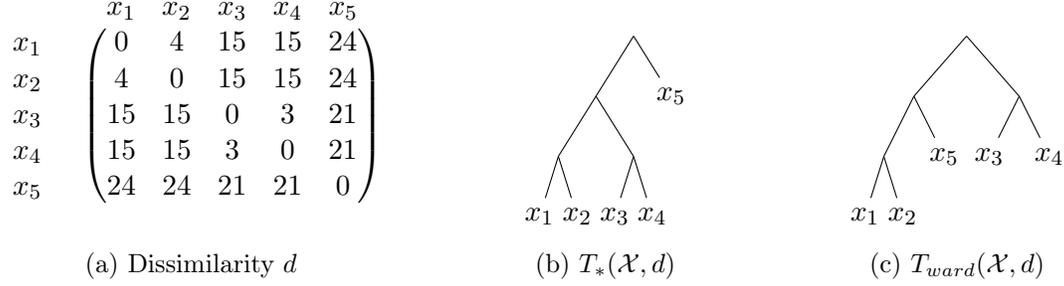

\subsection{Proof of Lemma~\ref{lemma:when_ultrametric}}
\label{subsec:proof_prop_when_ultrametric}

\begin{proof}[Proof of Lemma~\ref{lemma:when_ultrametric}]
We start by constructing a tree $T$ and a function $w \colon T \rightarrow \R_+$ such that $s(x, y) = w(\lca_{T}(x, y))$ for any $x, y \in \cX^2$, and $w(t) > w(t')$ for any $t \subsetneq t'$. Notice this is possible because $s$ is ultrametric. 

Let us show that $T$ is a valid hierarchy by proving that every vertex $t \in T$ satisfies Condition~\eqref{eq:hierarchical_tree}. Notice that for any vertex $t \in T$ and $x, y \in t$ and $z \in \cX \backslash t$,
\begin{align*}
 s(x,y) - s(x,z) 
 \weq w(\lca_{T}(x, y)) - w(\lca_{T}(x,z))
 \wge w(t) - w(\parent(t)) \ > \ 0,
\end{align*}
where we recall that $\parent(t)$ is the parent of vertex $t \in T$. 
Therefore, 
\begin{align*}
 \min_{x, y \in t, z \in \cX \backslash t} s(x,y) - s(x, z) 
 \ > \ 0,
\end{align*}
for any vertex $t \in T$, and thus $T$ is a valid hierarchy.

Theorem~\ref{theorem:algo_recover} states the most informative hierarchy contains all valid hierarchies; thus, $T \subseteq \trueHierarchy(\cX,s)$. Therefore, to show $T$ is the most informative hierarchy, it is enough to prove that there is no $t' \in \powerset(\cX) \backslash T$ such that $T \cup \{ t'\}$ is a valid hierarchy. We prove this by contradiction.

To be a valid hierarchy, first, $T \cup \{ t'\}$ needs to be a tree. To satisfy condition~3 in Definition~\ref{definition:set_representation_tree}, first $T$ must have at least one non-binary fan-out vertex $t$ i.e., $t \in T$ such that $|\children_T(t)| \geq 3$. Now you can choose one such vertex $t$ and define $t'=\bigcup_{t_1 \in A} t_1$ with $A \subsetneq \children_T(t)$ and $|A|\geq 2$. 

Consider any $t'$ satisfying $\exists t \in T$ such that 
$|\children_T(t)| \geq 3$ and 
$t'=\bigcup_{t_1 \in A} t_1$ with $A \subsetneq \children_T(t)$ and $|A|\geq 2$.
Let be $T' = T \cup \{t'\}$. 
Notice that by the construction of $T$, 
\begin{align*}
    \forall t_1 \neq t_2 \in \children_T(t) \colon \quad s(t_1,t_2) \weq w(\lca_T(t_1,t_2)) \weq w(t).
\end{align*}
Because $\children_T'(t') \subsetneq \children_T(t)$, this implies 
\begin{align*}
    \forall t_1 \neq t_2 \in \children_T'(t') \colon \quad s(t_1,t_2) \weq w(\lca_T(t_1,t_2)) \weq w(t). 
\end{align*}
Hence,
\begin{align*}
    \min_{x,y \in t', z \in \cX \backslash t'} s(x,y)-s(x,z) 
    \weq \min_{x,y \in t', z \in t \backslash t'} s(x,y)-s(x,z)  
    \weq w(t) - w(t) 
    \weq 0,
\end{align*}
and $T'$ is not a valid hierarchy. This proves that $T$ is the most informative hierarchy. 
\end{proof}

\section{Additional Discussions}

\subsection{Valid Hierarchy: Alternative Definition}

 Let us finish the discussion by providing an alternative definition of a valid hierarchy. For a tree $T \in \cT(\cX)$, let us look at the following condition:   
\begin{align}
 \forall t \in T \colon \quad \min_{ x,y \in t } \,s(x,y) \ > \ \max_{ x \in t,  z \in \cX \backslash t } s(x,z).
 \label{eq:hierarchical_tree_alternative}
\end{align} 
By defining a valid hierarchy as a tree verifying this new condition, we can derive analogous results to Theorems~\ref{theorem:mostInformativeHierarchy} and~\ref{theorem:algo_recover} (with small modifications to Algorithm~\ref{algo:merging_vertices}). Because this new condition (\ref{eq:hierarchical_tree_alternative}) is more restrictive than~\eqref{eq:hierarchical_tree}, a tree verifying (\ref{eq:hierarchical_tree_alternative}) is also valid according to Definition~\ref{def:hierarchical_tree}. 


\subsection{Efficient Implementation of Algorithm~\ref{algo:merging_vertices}}
\label{sec:efficient_algo_2}

The vertices to trim are exactly the ones that violate Condition~\eqref{eq:hierarchical_tree}. A naive implementation consists of iterating over all $t \in \Tlinkage$ from the root to the leaves to check for vertices $t \in \Tlinkage$ violating Condition~\eqref{eq:hierarchical_tree}. Using the following lemma, we can avoid re-verifying the same condition several times. 

\begin{lemma}
\label{lemma:no_need_to_check_all_vertices}
Let $\parent(t)$ denote the parent of $t \in T$. Condition~\eqref{eq:hierarchical_tree} in Definition~\ref{def:validHierarchies} is equivalent to the following condition
\begin{align}
 \forall t \in T \colon \quad \min_{ \substack{ x, y \in t, \\ z \in \parent(t) \backslash t } } \,s(x,y) - s(x,z) \ > \ 0. 
 \label{eq:hierarchical_tree_simplified}
\end{align} 
\end{lemma}

\begin{proof}[Proof of Lemma~\ref{lemma:no_need_to_check_all_vertices}]
Let $T$ be a rooted tree defined on $\cX$. Let $t$ be any vertex on $T$. The claim of the lemma is equivalent to 
\begin{align}
 \bigcup_{t \in T} \left\{ (x, y, z) \colon x \in t, y \in t, z \in \left(\parent(t)\backslash t\right) \right\} 
 \weq 
 \bigcup_{t \in T} \left\{ (x, y, z) \colon x \in t, y \in t, z \in \cX \backslash t \right\}.
\label{eq:equivalent_set}
\end{align}
Because $\left(\parent(t)\backslash t\right) \subseteq \cX \backslash t$, we have 
\begin{align}
    \left\{ (x, y, z) \colon x \in t, y \in t, z \in  \left(\parent(t)\backslash t\right) \right\} 
    \ \subseteq \ \left\{ (x, y, z) \colon x \in t, y \in t, z \in \cX \backslash t \right\}, 
\label{eq:equivalent_set_inclusion_1}
\end{align}
and thus 
\begin{align*}
 \bigcup_{t \in T} \left\{ (x, y, z) \colon x \in t, y \in t, z \in \left(\parent(t)\backslash t\right) \right\} 
 \subseteq 
 \bigcup_{t \in T} \left\{ (x, y, z) \colon x \in t, y \in t, z \in \cX \backslash t \right\}.
\end{align*}
Now, denoting by $\parent^{m}(t)$ the parent of $t$ at the $m^{\text{th}}$ generation, i.e., 
\begin{align*} 
\parent^{0}(t) = t, \; \parent^{1}(t) = \parent(t), \; \parent^{2}(t) = \parent(\parent(t)), \; \ldots, \; \parent^{m}(t) = \parent(\parent^{m-1}(t)), 
\end{align*}
we can recast $\cX \backslash t = \bigcup_{m \geq 1} \left( \parent^{m}(t)\backslash \parent^{m-1}(t)\right)$, so that 
\begin{align}
 \left\{ (x, y, z) \colon x \in t, y \in t, z \in \cX \backslash t \right\} 
 & \weq \left\{ (x, y, z) \colon x \in t, y \in t, z \in \bigcup_{m \geq 1} \left(\parent^{m}(t)\backslash \parent^{m-1}(t)\right) \right\} \nonumber \\
 & \ \subseteq \ \bigcup_{m \geq 1} \left\{ (x, y, z) \colon x \in \parent^{m-1}(t), y \in \parent^{m-1}(t), z \in  \left(\parent^{m}(t)\backslash \parent^{m-1}(t)\right) \right\} \nonumber \\
 & \ \subseteq \ \bigcup_{t \in T} \left\{ (x, y, z) \colon x \in t, y \in t, z \in  \left(\parent(t)\backslash t\right) \right\}.
 \label{eq:equivalent_set_inclusion_2}
\end{align}
 Combining~\eqref{eq:equivalent_set_inclusion_1} and~\eqref{eq:equivalent_set_inclusion_2} yields Equation~\eqref{eq:equivalent_set}.
\end{proof}